\documentclass[ssy, preprint]{imsart}

\usepackage{amsfonts}
\usepackage{bm} 
\usepackage{mathrsfs}
\usepackage{subfig}
\usepackage{algorithm}
\usepackage{algorithmic}
\usepackage{amssymb}
\usepackage[pdftex]{color,graphicx}

\usepackage{url}
\usepackage{multirow}
\RequirePackage[OT1]{fontenc} 
\RequirePackage{amsthm,amsmath}
\RequirePackage[numbers]{natbib}
\RequirePackage[colorlinks,citecolor=blue,urlcolor=blue]{hyperref}

\startlocaldefs

\newtheorem{thm}{Theorem}[section]
\newtheorem{lem}[thm]{Lemma}
\newtheorem{prop}[thm]{Proposition}
\newtheorem{corollary}[thm]{Corollary}
\newtheorem{definition}[thm]{Definition}
\newtheorem{example}[thm]{Example}
\newtheorem{fact}[thm]{Fact}

\newcommand\Jmin{J_{\min}}
\newcommand\Jmax{J_{\max}}
\newcommand\alg{CondST}
\newcommand\algpre{CondST\_Pre}
\newcommand\Tsaw{T_{saw}(i;G)}
\newcommand\Gp{\mathcal{G}(p, \frac{c}{p})}
\newcommand\prealpha{\beta}

\newcommand\E{\mathrm{E}}
\newcommand{\er}{Erd\H os-R\'enyi }
\newcommand\1[1]{\mathbb{I}_{\left\{#1\right\}}}

\endlocaldefs

\begin{document}

\begin{frontmatter}

\title{Learning Loosely Connected Markov Random Fields}
\runtitle{Learning Loosely Connected Markov Random Fields}


\author{\fnms{Rui} \snm{Wu}\thanksref{m1,t1}\ead[label=e1]{ruiwu1@illinois.edu}},
\author{\fnms{R} \snm{Srikant}\thanksref{m1}\ead[label=e2]{rsrikant@illinois.edu}}
\and
\author{\fnms{Jian} \snm{Ni}\thanksref{m2,t2}\ead[label=e3]{nij@ibm.com}}

\address{Department of Electrical and Computer Engineering\\
University of Illinois at Urbana-Champaign\\
Urbana, IL 61801, USA \\
\printead{e1}\\
\phantom{E-mail:\ }\printead*{e2}}

\address{IBM T. J. Watson Research Center\\
Yorktown Heights, NY 10598, USA\\
\printead{e3}}

\affiliation{University of Illinois at Urbana-Champaign\thanksmark{m1} and IBM T. J. Watson Research Center\thanksmark{m2}}

\thankstext{t1}{Research supported in part by AFOSR MURI FA 9550-10-1-0573.}
\thankstext{t2}{Jian Ni's work was done when he was at the University of Illinois at Urbana-Champaign.}

\runauthor{R. Wu, R. Srikant and J. Ni}

\begin{abstract}
We consider the structure learning problem for graphical models that we call loosely connected Markov random fields, in which the number of short paths between any pair of nodes is small, and present a new conditional independence test based algorithm for learning the underlying graph structure. The novel maximization step in our algorithm ensures that the true edges are detected correctly even when there are short cycles in the graph. The number of samples required by our algorithm is $C\log p$, where $p$ is the size of the graph and the constant $C$ depends on the parameters of the model. We show that several previously studied models are examples of loosely connected Markov random fields, and our algorithm achieves the same or lower computational complexity than the previously designed algorithms for individual cases. We also get new results for more general graphical models, in particular, our algorithm learns general Ising models on the \er random graph $\Gp$ correctly with running time $O(np^5)$.

\end{abstract}



\end{frontmatter}


\section{Introduction}\label{sec:intro}

 \nocite{*}

In many models of networks, such  as social networks and gene regulatory networks, each
node in the network represents a random variable and the graph encodes the conditional
independence relations among the random variables. A Markov random field is a
particular such representation which has
applications in a variety of areas (see \cite{Anima2} and the references therein).
In a Markov random field, the lack of an edge between two nodes implies that the two random variables
are independent, conditioned on all the other random variables in the network.

Structure learning, i.e, learning the underlying graph structure of a Markov random field, refers to
the problem of determining if there is an edge
between each pair of nodes, given i.i.d. samples from the joint distribution of the random
vector. As a concrete example of structure learning, consider a social network in which only the participants' actions are observed.
In particular, we do not observe or are unable to observe, interactions between the participants.
Our goal is to infer relationships among the nodes (participants) in such a network by understanding the correlations among the nodes.
The canonical example used to illustrate such inference problems is the US Senate \cite{Senate}. Suppose one has access to the voting patterns of the
senators over a number of bills (and not their party affiliations or any other information), the question we would like to answer is the
following: can we say that a particular senator's vote is independent of everyone else's when conditioned on a few other senators' votes?
In other words, if we view the senators' actions as forming a Markov Random Field (MRF), we want to infer the topology of the underlying graph.

In general, learning high dimensional densely connected graphical models requires large number of samples, and is usually computationally intractable. In this paper, we focus on a more tractable family which we call loosely connected MRFs. Roughly speaking, a Markov random field is loosely connected if the number of short paths between any pair of nodes is small. We show that many previously studied models are examples of this family. In fact, as densely connected graphical models are difficult to learn, some sparse assumptions are necessary to make the learning problem tractable. Common assumptions include an upper bound on the node degree of the underlying graph \cite{Guy,Sujay}, restrictions on the class of parameters of the joint probability distribution of the random variables to ensure correlation decay \cite{Guy,Sujay,Anima1}, lower bounds on the girth of the underlying graph \cite{Sujay}, and a sparse, probabilistic structure on the underlying random graph \cite{Anima1}. In all these cases, the resulted MRFs turn out to be loosely connected. In this sense, our definition here provides a unified view of the assumptions in previous works.

However, loosely connected MRFs are not always easy to learn. Due to the existence of short cycles, the dependence over an edge connecting a pair of neighboring nodes can be approximately cancelled by some short non-direct paths between them, in which case correctly detecting this edge is difficult, as shown in the following example. This example is perhaps well-known, but we present it here to motivate our algorithm presented later.
\begin{example}\label{example:cycle}
  Consider three binary random variables $X_i\in \{0, 1\}, i = 1, 2, 3$. Assume $X_1, X_2$ are independent $\text{Bernoulli}(\frac{1}{2})$ random variables and $X_3 = X_1 \oplus X_2$ with probability $0.9$, where $\oplus$ means exclusive or. We note that this joint distribution is symmetric, i.e., we get the same distribution if we assume that $X_2, X_3$ are independent $\text{Bernoulli}(\frac{1}{2})$ and $X_1 = X_2\oplus X_3$ with probability $0.9$. Therefore, the underlying graph is a triangle. However, it is not hard to see that the three random variables are marginally independent. For this simple example, previous methods in \cite{Sujay, Anima2} fail to learn the true graph.\qed
\end{example}

We propose a new algorithm that correctly learns the graphs for loosely connected MRFs. For each node, the algorithm loops over all the other nodes to determine if they are neighbors of this node. The key step in the algorithm is a max-min conditional independence test, in which the maximization step is designed to detect the edges while the minimization step is designed to detect non-edges. The minimization step is used in several previous works such as \cite{Anima1, Anima2}. The maximization step has been added to explicitly break the short cycles that can cause problems in edge detection. If the direct edge is the only edge between a pair of neighboring nodes, the dependence over the edge can be detected by a simple independence test. When there are other short paths between a pair of neighboring nodes, we first find a set of nodes that separates all the non-direct paths between them, i.e., after removing this set of nodes from the graph, the direct edge is the only short path connecting to two nodes. Then the dependence over the edge can again be detected by a conditional independence test where the conditioned set is the set above. In Example~\ref{example:cycle}, $X_1$ and $X_3$ are unconditionally independent as the dependence over edge $(1, 3)$ is canceled by the other path $(1, 2, 3)$. If we break the cycle by conditioning on $X_2$, $X_1$ and $X_3$ become dependent, so our algorithm is able to detect the edges correctly.
As the size of the conditioned sets is small for loosely connected MRFs, our algorithm has low complexity. In particular, for models with at most $D_1$ short paths between non-neighbor nodes and $D_2$ non-direct paths between neighboring nodes, the running time for our algorithm is $O(np^{D_1+D_2+2})$.

If the MRF satisfies a pairwise non-degeneracy condition, i.e., the correlation between any pair of neighboring nodes is lower bounded by some constant, then we can extend the basic algorithm to incorporate a correlation test as a preprocessing step. For each node, the correlation test adds those nodes whose correlation with the current node is above a threshold to a candidate neighbor set, which is then used as the search space for the more computationally expensive max-min conditional independence test. If the MRF has fast correlation decay, the size of the candidate neighbor set can be greatly reduced, so we can achieve much lower computational complexity with this extended algorithm.

When applying our algorithm to Ising models, we get lower computational complexity for a ferromagnetic Ising model than a general one on the same graph. Intuitively, the edge coefficient $J_{ij}>0$ means that $i$ and $j$ are positively dependent. For any path between $i, j$, as all the edge coefficients are positive, the dependence over the path is also positive. Therefore, the non-direct paths between a pair of neighboring nodes $i, j$ make $X_i$ and $X_j$, which are positively dependent over the edge $(i, j)$, even more positively dependent. Therefore, we do not need the maximization step which breaks the short cycles and the resulting algorithm has running time $O(np^{D_1+2})$. In addition, the pairwise non-degeneracy condition is automatically satisfied and the extended algorithm can be applied.

\subsection{Relation to Prior Work}

We focus on computational complexity rather than sample complexity in comparing our algorithm with previous algorithms. In fact, it has been shown that $\Omega(\log p)$ samples are required to learn the graph correctly with high probability, where $p$ is the size of the graph \cite{Sample_Complexity}. For all the previously known algorithms for which analytical complexity bounds are available, the number of samples required to recover the graph correctly with high probability, i.e, the sample complexity, is $O(\log p)$. Not surprisingly, the sample complexity for our algorithm is also $O(\log p)$ under reasonable assumptions. 

Our algorithm with the probability test reproduces the algorithm in \cite[Theorem 3]{Guy} for MRFs on bounded degree graphs. Our algorithm is more flexible and achieves lower computational complexity for MRFs that are loosely connected but have a large maximum degree. In particular, reference \cite{Sujay} proposed a low complexity greedy algorithm that is correct when the MRF has correlation decay and the graph has large girth. We show that under the same assumptions, we can first perform a simple correlation test and reduce the search space for neighbors from all the nodes to a constant size candidate neighbor set. With this preprocessing step, our algorithm and the algorithms in \cite{Guy, Sujay, Sujay2} have computational complexity $O(np^2)$, which is lower than what we would get by only applying the greedy algorithm \cite{Sujay}. The results in \cite{Sujay2} improve over \cite{Sujay} by proposing two new greedy algorithms that are correct for learning small girth graphs. However, the algorithm in \cite{Sujay2} requires a constant size candidate neighbor set as input, which might not be easy to obtain in general. In fact, for MRFs with bad short cycles as in Example~\ref{example:cycle}, learning a candidate neighbor set can be as difficult as directly learning the neighbor set.

Our analysis of the class of Ising models on sparse \er random graphs $\Gp$ was motivated by the results in \cite{Anima1} which studies the special case of the so-called ferromagnetic Ising models defined over an Erd\H os-R\'enyi random graph. The computational complexity of the algorithm in \cite{Anima1} is $O(np^4)$. In this case, the key step of our algorithm reduces to the algorithm in \cite{Anima1}. But we show that, under the ferromagnetic assumption, we can again perform a correlation test to reduce the search space for neighbors, and the total computational complexity for our algorithm is $O(np^2)$.

The results in \cite{Anima2} extend the results in \cite{Anima1} to general Ising models and more general sparse graphs (beyond the Erd\H os-R\'enyi model). We note that the tractable graph families in \cite{Anima2} is similar to our notion of loosely-connected MRFs.
For general Ising models over sparse \er random graphs, our algorithm has computational complexity $O(np^5)$ while the algorithm in \cite{Anima2} has computational complexity $O(np^4)$. The difference comes from the fact that our algorithm has an additional maximization step to break bad short cycles as in Example~\ref{example:cycle}. Without this maximization step, the algorithm in \cite{Anima2} fails for this example. The performance analysis in \cite{Anima2} explicitly excludes such difficult cases by noting that these ``unfaithful'' parameter values have Lebesgue measure zero \cite[Section B.3.2]{Anima2}. However, when the Ising model parameters lie close to this Lebesgue measure zero set, the learning problem is still ill posed for the algorithm in \cite{Anima2}, i.e., the sample complexity required to recover the graph correctly with high probability depends on how close the parameters are to this set, which is not the case for our algorithm. In fact, the same problem with the argument that the unfaithful set is of Lebesgue measure zero has been observed for causal inference in the Gaussian case \cite{Faithful}. It has been shown in \cite{Faithful} that a stronger notion of faithfulness is required to get uniform sample complexity results, and the set that is not strongly faithful has non-zero Lebesgue measure and can be be surprisingly large.

Another way to learn the structures of MRFs is by solving $l_1$-regularized convex optimizations under a set of incoherence conditions \cite{Martin}. It is shown in \cite{Jose} that, for some Ising models on a bounded degree graph, the incoherence conditions hold when the Ising model is in the correlation decay regime. But the incoherent conditions do not have a clear interpretation as conditions for the graph parameters in general and are NP-hard to verify for a given Ising model \cite{Jose}. 
Using results from standard convex optimization theory \cite{Boyd}, it is possible to design a polynomial complexity algorithm to approximately solve the $l_1$-regularized optimization problem. However, the actual complexity will depend on the details of the particular algorithm used, therefore, it is not clear how to compare the computational complexity of our algorithm with the one in \cite{Martin}.

We note that the recent development of directed information graphs \cite{Negar} is closely related to the theory of MRFs. Learning a directed information graph, i.e., finding the causal parents of each random process, is essentially the same as finding the neighbors of each random variable in learning a MRF. Therefore, our algorithm for learning the MRFs can potentially be used to learn the directed information graphs as well.

The paper is organized as follows. We present some preliminaries in the next section. In Section~\ref{sec:loose}, we define loosely-connected MRFs and show that several previously studied models are examples of this family. In Section~\ref{sec:alg}, we present our algorithm and show the conditions required to correctly recover the graph. We also provide the concentration results in this section. In Section~\ref{sec:complexity_general}, we apply our algorithm to the general Ising models studied in Section~\ref{sec:loose} and evaluate its sample complexity and computational complexity in each case. In Section~\ref{sec:complexity_ferro}, we show that our algorithm achieves even lower computational complexity when the Ising model is ferromagnetic. Experimental results are presented in Section~\ref{sec:experiment}.

\section{Preliminaries}\label{sec:prelim}

\subsection{Markov Random Fields (MRFs)}

Let $X = (X_1, X_2, \dots, X_p)$ be a random vector with distribution $P$ and $G = (V,
E)$ be an undirected graph consisting of $|V| = p$ nodes with each node $i$ associated with
the $i^{\rm th}$ element $X_i$ of $X.$ Before we define an MRF, we introduce the notation $X_S$ to denote any subset $S$ of the random variables in $X.$
A random vector and graph pair $(X, G)$ is called an MRF
if it satisfies one of the following three Markov properties:
\begin{enumerate}
  \item Pairwise Markov: $X_i\perp X_j|X_{V\setminus \{i, j\}}, \forall (i, j)\not\in E,$ where $\perp$ denotes independence.
  \item Local Markov: $X_i\perp X_{V\setminus \{i\cup N_i\}}|X_{N_i}, \forall i\in V,$ where $N_i$ is the set of neighbors of node $i.$
  \item Global Markov: $X_A\perp X_B|X_S$, if $S$ separates $A, B$ on $G$. In this case,
  we say $G$ is an \emph{I-map} of $X$. Further if $G$ is an I-map of $X$ and the global Markov property does not hold if any edge of $G$ is removed, then
  $G$ is called a \emph{minimal I-map} of X.
\end{enumerate}
In all three cases, $G$ encodes a subset of the conditional independence relations of $X$ and we
say that $X$ is Markov with respect to $G$. We note that the global Markov property implies the local Markov property, which in turn
implies the pairwise Markov property.

When $P(x)>0, \forall x$, the three Markov properties are equivalent, i.e., if there exists a $G$ under which one of the Markov properties is satisfied, then the other two are also satisfied.
Further, in the case when $P(x)>0, \forall x,$ there exists a unique minimal
I-map of $X$. The unique minimal I-map $G = (V, E)$ is constructed as follows:
\begin{enumerate}
  \item Each random variable $X_i$ is associated with a node $i\in V.$
  \item $(i, j)\not\in E$ if and only if $X_i\perp X_j | X_{V\setminus\{i, j\}}$.
\end{enumerate}

In this case, we consider the case $P(x)>0, \forall x$ and are interested in learning the structure of the associated unique minimal I-map.
We will also assume that, for each $i,$ $X_i$ takes on values in a discrete, finite set $\mathcal{X}$. We will also be interested in the special case where the MRF is an Ising model, which we describe next.

\subsection{Ising Model}
Ising models are a type of well-studied pairwise Markov random fields. In an Ising model,
each random variable $X_i$ takes values in the set $\mathcal{X} = \{-1, +1\}$ and the
joint distribution is parameterized by constants called edge coefficients $J$ and external fields $h:$
\begin{align*}
  P(x) = \frac{1}{Z}\exp\Big(\sum_{(i, j)\in E}{J_{ij}x_ix_j}+\sum_{i\in V}h_ix_i\Big).
\end{align*}
where $Z$ is a normalization constant to make $P(x)$ a probability distribution.
If $h = 0$, we say the Ising model is zero-field. If $J_{ij}\geq 0$, we say the Ising
model is ferromagnetic.

Ising models have the following useful property. Given an Ising model,
the conditional probability $P(X_{V\setminus S}|x_S)$ corresponds to an
Ising model on $V\setminus S$ with edge coefficients $J_{ij}, i, j\in V\setminus S$
unchanged and modified external fields $h_i+h_i', i\in V\setminus S$, where $h_i' = \sum_{(i, j)\in E, j\in
S}J_{ij}x_j$ is the additional external field on node $i$ induced by fixing $X_S = x_S$.

\subsection{Random Graphs}

A random graph is a graph generated from a prior distribution over the set of all possible graphs with a given number of nodes.
Let $\chi_p$ be a function on graphs with $p$ nodes and let $C$ be a constant.
We say $\chi_p \geq C$ almost always for a family of random graphs indexed by $p$ if $P(\chi_p\geq C)\to
1$ as $p\to \infty$. Similarly, we say $\chi_p\to C$ almost always for a family of random graphs if $\forall \epsilon>0, P(|\chi_p - C|>\epsilon)\to
1$ as $p\to \infty.$ This is a slight variation of the definition of almost always in \cite{Alon}.

The Erd\H os-R\'enyi random graph $\Gp$ is a graph on $p$ nodes
in which the probability of an edge being in the graph is $\frac{c}{p}$ and the edges are
generated independently. We note that, in this random graph, the average degree of a node
is $c$. In this paper, when we consider random graphs, we only consider the Erd\H os-R\'enyi random graph $\Gp.$

\subsection{High-Dimensional Structure Learning}

In this paper, we are interested in inferring the structure of the graph $G$ associated with an MRF
$(X,G).$ We will assume that $P(x)>0, \forall x,$ and $G$ will refer to the corresponding unique minimal I-map. The goal of structure learning is to design an algorithm that,
given $n$ i.i.d. samples $\{X^{(k)}\}_{k =1}^n$ from the distribution $P,$ outputs an estimate $\hat{G}$
which equals $G$ with high probability when $n$ is large. We say that two graphs are equal when their node and edge sets are identical.

In the classical setting, the accuracy of estimating $G$ is considered only when the
sample size $n$ goes to infinity while the random vector dimension $p$ is held fixed.
This setting is restrictive for many contemporary applications, where the problem size $p$ is
much larger than the number of samples.
A more suitable assumption allows both $n$ and $p$ to become large, with $n$ growing at a slower rate than $p.$ In such a case, the structure learning problem is said to be high-dimensional.

An algorithm for structure learning is evaluated both by its computational complexity and
sample complexity. The computational complexity refers to the number of computations required to
execute the algorithm, as a function of $n$ and $p.$ When $G$ is a deterministic graph, we say the
algorithm has sample complexity $f(p)$ if, for $n = O(f(p)),$
there exist constants $c$ and $\alpha>0,$ independent of $p,$ such that
$\Pr(\hat{G} = G)\geq 1-\frac{c}{p^{\alpha}}$ for all $P$ which are Markov with respect to $G.$
When $G$ is a random graph drawn from some prior distribution, we say the
algorithm has sample complexity $f(p)$ if the above is true almost always.
In the high-dimensional setting $n$ is much smaller than $p.$ In fact, we will show that, for the algorithms
described in this paper, $f(p)=\log p.$

\section{Loosely Connected MRFs}\label{sec:loose}

Loosely connected Markov random fields are undirected graphical models in which the number of short paths between any pair of nodes is small.
Roughly speaking, a path between two nodes is short if the dependence between two node is non-negligible even if all other paths between the nodes are removed. Later, we will more precisely quantify the term "short" in terms of the correlation decay property of the MRF. For simplicity, we say that a set $S$ separates some paths between nodes $i$ and $j$ if removing $S$ disconnects these paths.
In such a graphical model, if $i, j$ are not neighbors, there is a small set of nodes $S$ separating all the short paths between them, and conditioned on this set of variables $X_S$ the two variables $X_i$ and $X_j$ are approximately independent. On the other hand, if $i, j$ are neighbors, there is a small set of nodes $T$ separating all the short non-direct paths between them, i.e, the direct edge is the only short path connecting the two nodes after removing $T$ from the graph. Conditioned on this set of variables $X_T$, the dependence of $X_i$ and $X_j$ is dominated by the dependence over the direct edge hence is bounded away from zero.
The following necessary and sufficient condition for the non-existence of an edge in a graphical model shows that both the sets $S$ and $T$ above are essential for learning the graph, which we have not seen in prior work.
\begin{lem}\label{lem:edge}
Consider two nodes $i$ and $j$ in $G.$ Then, $(i, j)\not\in E$ if and only if $\exists S, \forall T, X_i\perp X_j |X_S, X_T$.
\end{lem}
\begin{proof}
Recall from the definition of the minimal I-map that
$(i, j)\not\in E$ if and only if $X_i\perp X_j | X_{V\setminus\{i, j\}}$. Therefore, the statement of the lemma is equivalent to
  \begin{align*}
    I(X_i;X_j|X_{V\setminus\{i, j\}}) = 0 \Leftrightarrow \min_S\max_T I(X_i; X_j|X_S, X_T) = 0,
  \end{align*}
where $I(X_i;X_j|X_S)$ denotes the mutual information between $X_i$ and $X_j$ conditioned on $X_S,$
and we have used the fact that $X_i\perp X_j|X_S$ is equivalent to $I(X_i;X_j|X_S)=0.$
  Notice that $$\min_S\max_T I(X_i; X_j|X_S, X_T) = \min_S\max_{T'\supset S} I(X_i; X_j|X_{T'})$$
  and $\max_{T'\supset S} I(X_i; X_j|X_{T'})$ is an increasing function in $S$. The minimization over
  $S$ is achieved at $S = V\setminus\{i, j\},$ i.e.,
  \begin{align*}
    I(X_i;X_j|X_{V\setminus\{i, j\}}) = \min_S\max_T I(X_i; X_j|X_S, X_T).
  \end{align*}
\end{proof}

This lemma tells that, if there is not an edge between node $i$ and $j$, we can find a set of nodes $S$ such that the removal of S from the graph separates $i$ and $j$. From the global Markov property, this implies that $X_i\perp X_j|X_S$. However, as Example~\ref{example:cycle} shows, the converse is not true. In fact, for $S$ being the empty set or $S = \emptyset$, we have $X_1\perp X_2|X_S$, but $(1, 2)$ is indeed an edge in the graph. The above lemma completes the statement in the converse direction, showing that we should also introduce a set $T$ in addition to the set $S$ to correctly identify the edge.

Motivated by this lemma, we define loosely connected MRFs as follows.
\pagebreak

\noindent\hrulefill
\begin{definition}\label{def:loose}
  We say a MRF is $(D_1, D_2, \epsilon)$-loosely connected if
  \begin{enumerate}
    \item for any $(i, j)\not\in E$, $\exists S$ with $|S|\leq D_1$, $\forall T$ with $|T|\leq D_2$, $$\Delta(X_i; X_j|X_S, X_T)\leq \frac{\epsilon}{4},$$
    \item for any $(i, j)\in E$, $\forall S$ with $|S| \leq D_1$ , $\exists T$ with $|T|\leq D_2$, $$\Delta(X_i; X_j|X_S, X_T)\geq \epsilon,$$
  \end{enumerate}
  for some conditional independence test $\Delta$.
\end{definition}
\noindent\hrulefill
\ \\

The conditional independence test $\Delta$ should satisfy $\Delta(X_i; X_j|X_S, X_T) = 0$
if and only if $X_i\perp X_j|X_S, X_T$. In this paper, we use two types of
conditional independence tests:
\begin{itemize}
  \item Mutual Information Test: $$\Delta(X_i; X_j|X_S, X_T) = I(X_i; X_j|X_S,
    X_T).$$
  \item Probability Test: $$\Delta(X_i; X_j|X_S, X_T) = \max_{x_i, x_j, x_j', x_S, x_T}|P(x_i | x_j, x_S, x_T)-P(x_i|x_j', x_S,
      x_T)|.$$
\end{itemize}
Later on, we will see that the probability test gives lower sample complexity for learning Ising models on bounded degree graphs, while the mutual information test gives lower sample complexity for learning Ising models on graphs with unbounded degree.

Note that the above definition restricts the size of the sets $S$ and $T$ to make the learning problem tractable. We show in the rest of the section that several important Ising models are examples of loosely connected MRFs. Unless otherwise stated, we assume that the edge coefficients $J_{ij}$ are bounded, i.e., $\Jmin\leq |J_{ij}| \leq \Jmax$.

\subsection{Bounded Degree Graph}
We assume the graph has maximum degree $d$. For any $(i, j)\not\in E$, the set $S = N_i$ of size at most $d$ separates $i$ and $j$, and for any set $T$ we have $\Delta(X_i; X_j|X_S, X_T) = 0$. For any $(i, j)\in E$, the set $T = N_i\setminus j$ of size at most $d-1$ separates all the non-direct paths between $i$ and $j$. Moreover, we have the following lower bound for neighbors from \cite[Proposition
2]{Guy}.

\begin{prop}\label{prop:bdd_loose}
  When $i, j$ are neighbors and $T = N_i\setminus j$, there is a choice of $x_i, x_j,
  x_j', x_S, x_T$ such that
  \begin{align*}
  |P(x_i|x_j, x_S, x_T) - P(x_i|x_j', x_S, x_T)|\geq
  \frac{\tanh(2J_{\min})}{2e^{2J_{\max}}+2e^{-2J_{\max}}}\triangleq \epsilon.
\end{align*}
\qed
\end{prop}

Therefore, the Ising model on a bounded degree graph with maximum degree $d$ is a $(d, d-1, \epsilon)$-loosely connected MRF. We note that here we do not use any correlation decay property, and we view all the paths as short.

\subsection{Bounded Degree Graph, Correlation Decay and Large Girth}
In this subsection, we still assume the graph has maximum degree $d$. From the previous subsection, we already know that the Ising model is loosely connected. But we show that when the Ising model is in the correlation decay regime and further has large girth, it is a much sparser model than the general bounded degree case.

Correlation decay is a property of MRFs which says that, for any pair of nodes $i, j$, the
correlation of $X_i$ and $X_j$ decays with the distance between $i, j$. When a MRF has
correlation decay, the correlation of $X_i$ and $X_j$ is mainly determined by the short paths between nodes $i,
j$, and the contribution from the long paths is negligible. It
is known that when $\Jmax$ is small compared with $d,$ the Ising model has correlation
decay. More specifically, we have the following lemma, which is a consequence of the strong correlation decay property \cite[Theorem 1]{Correlation_Decay}.

\begin{lem}\label{lem:bdd_correlation_decay}
  Assume $(d-1)\tanh{J_{\max}}<1$. $\forall i, j\in V, d(i, j) = l$, then for any set $S$ and $\forall x_i,
  x_j, x_j', x_S$,
  \begin{align*}
    |P(x_i|x_j, x_S)-P(x_i|x_j', x_S)|\leq
    4J_{\max}d[(d-1)\tanh{J_{\max}}]^{l-1} \triangleq \prealpha\alpha^l,
  \end{align*}
  where $\prealpha = \frac{4J_{\max}d}{(d-1)\tanh{J_{\max}}}$ and $\alpha = (d-1)\tanh{J_{\max}}$.
\end{lem}
\begin{proof}
  For some given $x_i, x_j, x_j', x_S$, w.l.o.g. assume $P(x_i|x_j, x_S)\geq P(x_i|x_j', x_S)$. Applying the \cite[Theorem 1]{Correlation_Decay} with $\Lambda = \{j\}\cup S$, we get
  \begin{align*}
    |P(x_i|x_j, x_S)-P(x_i|x_j', x_S)|\leq & 1-\frac{P(x_i|x_j', x_S)}{P(x_i|x_j, x_S)}\\
    \leq & 1-e^{-4J_{\max}d[(d-1)\tanh{J_{\max}}]^{d(i, j)-1}}\\
    \leq & 4J_{\max}d[(d-1)\tanh{J_{\max}}]^{d(i, j)-1}.
  \end{align*}
\end{proof}

This lemma implies that, in the correlation decay regime $(d-1)\tanh{J_{\max}}<1$, the
Ising model has exponential correlation decay, i.e., the correlation between a pair of
nodes decays exponentially with their distance. We say that a path of length $l$ is short if $\beta\alpha^l$ is above some desired threshold.

The girth of a graph is defined as the length of the shortest cycle in the graph, and large girth implies that there is no short cycle in the graph.
When the Ising model is in the correlation decay regime and the girth of the graph is large in terms of the correlation decay parameters, there is at most one
short path between any pair of non-neighbor nodes, and no short paths other than the direct
edge between any pair of neighboring nodes. Naturally, we can use $S$ of size 1 to
approximately separate any pair of non-neighbor nodes and do not need $T$ to block the
other paths for neighbor nodes as the correlations are mostly due to the direct edges.
Therefore, we would expect this Ising model to be $(1, 0, \epsilon)$-loosely connected for some constant $\epsilon$. In fact, the following theorem gives an explicit characterization of $\epsilon$. The condition on the girth below is chosen such that there is at most one short path between any pair of nodes, so a path is called short if it is shorter than half of the girth.

\begin{thm}\label{thm:girth_loose}
  Assume $(d-1)\tanh{J_{\max}}<1$ and the girth $g$ satisfies 
  \begin{align*}
    \beta \alpha^{\frac{g}{2}}\leq A\wedge \ln 2,
  \end{align*}
  where $A = \frac{1}{1800}(1-e^{-4\Jmin})e^{-8dJ_{\max}}$. Let $\epsilon =
  48Ae^{4dJ_{\max}}$. Then $\forall (i, j)\in E$,
  $$
    \displaystyle\min_{\substack{S\subset V\setminus \{i\cup j\}\\|S|\leq D_1}}\ \max_{x_i, x_j, x_j', x_S}|P(x_i|x_j, x_S)-P(x_i|x_j', x_S)|>\epsilon,
  $$
  and $\forall (i, j)\notin E$,
  $$
    \displaystyle\min_{\substack{S\subset V\setminus \{i\cup j\}\\|S|\leq D_1}}\ \max_{x_i, x_j, x_j', x_S}|P(x_i|x_j, x_S)-P(x_i|x_j', x_S)|\leq\frac{\epsilon}{4}.
  $$
\end{thm}
\begin{proof}
  See Appendix~\ref{appendix:bdd}.
\end{proof}

\subsection{\er Random Graph $\Gp$ and Correlation Decay}

We assume the graph $G$ is generated from the prior $\mathcal{G}(p,
\frac{c}{p})$ in which each edge is in $G$ with probability $\frac{c}{p}$ and the average
degree for each node is $c$. For this random graph, the maximum degree scales as $O(\frac{\ln p}{\ln\ln p})$ with high
probability \cite{Alon}. Thus, we cannot use the results for bounded degree graphs even though the average
degree remains bounded as $p\rightarrow\infty.$

It is known from prior work \cite{Anima1} that, for ferromagnetic Ising models, i.e, $J_{ij}\geq 0$ for any $i$ and $j$, when $\Jmax$ is small compared with the average degree $c$, the random
graph is in the correlation decay regime and the number of short paths between any pair of nodes is at most 2 asymptotically. We show that the same result holds for general Ising models. Our proof is related to the techniques developed in \cite{Anima1}, but certain steps in the proof of \cite{Anima1} do rely on the fact that the Ising model is ferromagnetic, so the proof does not directly carry over. We point out similarities and differences as we proceed in Appendix~\ref{appendix:random}.

More specifically, letting $\gamma_p = \frac{\log p}{K\log c}$ for some $K\in(3, 4)$, the following theorem shows that nodes that are at least $\gamma_p$ hops from each other have negligible impact on each other. As a consequence of the following theorem, we can say that a path is short if it is at most $\gamma_p$ hops.

\begin{thm}\label{thm:random_correlation_decay}
  Assume $\alpha = c\tanh\Jmax<1$. Then, the following properties are true almost always. \\
  (1) Let $G$ be a graph generated from the prior $\mathcal{G}(p,
  \frac{c}{p}).$ If $i, j$ are not
  neighbors in $G$ and $S$ separates all the paths shorter than $\gamma_p$ hops between $i, j$, then $\forall x_i, x_j, x_j', x_S$,
  \begin{align*}
    |P(x_i|x_j, x_S)-P(x_i|x_j', x_S)|\leq |B(i, \gamma_p)|(\tanh\Jmax)^{\gamma_p}= o(p^{-\kappa}),
  \end{align*}
  for all Ising models $P$ on $G,$ where $\kappa = \frac{\log \frac{1}{\alpha}}{4\log c}$ and $B(i, \gamma_p)$ is the set of all nodes which are at most $\gamma_p$ hops away from $i.$.\\
  (2) There are at most two paths shorter than $\gamma_p$ between any pair of nodes.
\end{thm}
\begin{proof}
  See Appendix~\ref{appendix:random}.
\end{proof}

The above result suggests that for Ising models on the random graph there are at most two short paths between non-neighbor nodes and one short non-direct path between neighboring nodes, i.e., it is a $(2, 1, \epsilon)$-loosely connected MRF. Further the next two theorems prove that such a constant $\epsilon$ exists. The proofs are in Appendix~\ref{appendix:random}.

\begin{thm}\label{thm:random_loose_nonneighbor}
  For any $(i, j)\not\in E$, let $S$ be a set separating the paths shorter than $\gamma_p$ between $i,
  j$ and assume $|S|\leq 3$, then almost always
  \begin{align*}
    I(X_i;X_j|X_S) = o(p^{-2\kappa}).
  \end{align*}
  \qed
\end{thm}

\begin{thm}\label{thm:random_loose_neighbor}
  For any $(i, j)\in E$, let $T$ be a set separating the non-direct paths shorter than $\gamma_p$ between $i,
  j$ and assume $|T|\leq 3$, then almost always
  \begin{align*}
    I(X_i;X_j|X_T) = \Omega(1).
  \end{align*}
  \qed
\end{thm}

\section{Our Algorithm and Concentration results}\label{sec:alg}

Learning the structure of a graph is equivalent to learning if there exists an edge between every pair of nodes in the graph. Therefore, we would
like to develop a test to determine if there exists an edge between two nodes or not. From Definition~\ref{def:loose}, it should be clear that learning a loosely connected MRF is straightforward. For non-neighbor nodes, we search for the set $S$ that separates all the short paths between them, while for neighboring nodes, we search for the set $T$ that separates all the non-direct short paths between them. As the MRF is loosely connected, the size of the above sets are small, therefore the complexity of the algorithm is low.

Given $n$ i.i.d. samples $\{X^{(k)}\}_{k = 1}^n$ from the distribution the empirical distribution $\hat{P}$ is defined as follows. For any set $A$,
\begin{align*}
  \hat{P}(x_A) = \frac{1}{n}\sum_{i = 1}^n\1{X_A^{(i)} = x_A}.
\end{align*}
Let $\hat{\Delta}$ be the empirical conditional independence test which is the same as $\Delta$ but computed using $\hat{P}$. Our first algorithm is as follows.

\begin{algorithm}
  \caption{$\alg(D_1, D_2, \epsilon)$}
  \begin{algorithmic}
  \FOR {$i, j\in V$}
    \IF {$\exists S \text{ with } |S|\leq D_1, \forall T \text{ with } |T|\leq D_2,$
    $\hat{\Delta}(X_i; X_j|X_S, X_T)
    \leq \frac{\epsilon}{2}$\\}
    \STATE $(i, j)\not\in E$
    \ELSE \STATE $(i, j)\in E$
    \ENDIF
  \ENDFOR
  \end{algorithmic}
\end{algorithm}

For clarity, when we specifically use the mutual information test (or the probability test), we denote the corresponding algorithm
by $\alg_I$ (or $\alg_P$). When the empirical conditional independence test $\hat{\Delta}$ is close to the exact test $\Delta$, we immediately get the following result.

\begin{fact}\label{result:alg1}
  For a $(D_1, D_2, \epsilon)$-loosely connected MRF, if
  \begin{align*}
    |\hat{\Delta}(X_i; X_j|X_A)-\Delta(X_i; X_j|X_A)|<\frac{\epsilon}{4}
  \end{align*}
  for any node $i, j$ and set $A$ with $|A|\leq D_1+D_2$, then $\alg(D_1, D_2, \epsilon)$ recovers the graph correctly. The running time for the algorithm is $O(np^{D_1+D_2+2})$.
\end{fact}
\begin{proof}
  The correctness is immediate. We note that, for each pair of $i, j$ in $V$, we search $S, T$ in $V$. So the possible combinations of $(i, j, S, T)$ is $O(p^{D_1+D_2+2})$ and we get the running time result.
\end{proof}

When the MRF has correlation decay, it is possible to reduce the computational complexity by restricting the search space for the set $S$ and $T$ to a smaller candidate neighbor set. In fact, for each node $i$, the nodes which are a certain distance away from $i$ have small correlation with $X_i$. As suggested in \cite{Guy}, we can first perform a pairwise correlation test to eliminate these nodes from the candidate neighbor set of node $i$. To make sure the true neighbors are all included in the candidate set, the MRF needs to satisfy an additional pairwise non-degeneracy condition. Our second algorithm is as follows.

\begin{algorithm}
  \caption{$\algpre(D_1, D_2, \epsilon, \epsilon')$}
  \begin{algorithmic}
  \FOR {$i\in V$}
    \STATE $L_i = \{j\in V\setminus i, \displaystyle\max_{x_i, x_j, x_j'}|\hat{P}(x_i|
    x_j)-\hat{P}(x_i|x_j')|>\frac{\epsilon'}{2}\}$.
    \FOR {$j\in L_i$}
    \IF {$\exists S\subset L_i \text{ with } |S|\leq D_1, \forall T\subset L_i \text{ with } |T|\leq D_2, \hat{\Delta}(X_i; X_j|X_S, X_T)
    \leq \frac{\epsilon}{2}$} \STATE $j\notin N_i$
       \ELSE \STATE $j\in N_i$
       \ENDIF
    \ENDFOR
  \ENDFOR
  \end{algorithmic}
\end{algorithm}

The following result provides conditions under which the second algorithm correctly learns the MRF.

\begin{fact}\label{result:alg2}
  For a $(D_1, D_2, \epsilon)$-loosely connected MRF with
  \begin{align}
    \max_{x_i, x_j, x_j'}|P(x_i|x_j)-P(x_i|x_j')|>\epsilon' \label{equ:pairwise_non_deg}
  \end{align}
  for any $(i, j)\in E$, if
  \begin{align*}
    |\hat{P}(x_i|x_j)-P(x_i|x_j)|<\frac{\epsilon'}{8}
  \end{align*}
  for any node $i, j$ and $x_i, x_j$, and
  \begin{align*}
    |\hat{\Delta}(X_i; X_j|X_A)-\Delta(X_i; X_j|X_A)|<\frac{\epsilon}{4}
  \end{align*}
  for any node $i, j$ and set $A$ with $|A|\leq D_1+D_2$, then $\algpre(D_1, D_2, \epsilon, \epsilon')$ recovers the graph correctly. Let $L = \max_i|L_i|$. The running time for the algorithm is $O(np^2+npL^{D_1+D_2+1})$.
\end{fact}
\begin{proof}
  By the pairwise non-degeneracy condition (\ref{equ:pairwise_non_deg}), the neighbors of node $i$ are all included in the candidate neighbor set $L_i$. We note that this preprocessing step excludes the nodes whose correlation with node $i$ is below $\frac{\epsilon'}{4}$.
  Then in the inner loop, the correctness of the algorithm is immediate. The running time of the correlation test is $O(np^2)$. We note that, for each $i$ in $V$, we loop over $j$ in $L_i$ and search $S$ and $T$ in $L_i$. So the possible combinations of $(i, j, S, T)$ is $O(pL^{D_1+D_2+1})$. Combining the two steps, we get the running time of the algorithm.
\end{proof}

Note that the additional non-degeneracy condition (\ref{equ:pairwise_non_deg}) required for the second algorithm to execute correctly is not satisfied for all graphs (recall Example~\ref{example:cycle}).

\subsection{Concentration Results}

In this subsection, we show a set of concentration results for the empirical quantities in the above algorithm for general discrete MRFs, which will be used to obtain the sample complexity results in Section~\ref{sec:complexity_general} and Section~\ref{sec:complexity_ferro}.

\begin{lem}\label{lem:concentration}
  Fix $\gamma>0$. Let $L = \max_i|L_i|$. For $\forall \alpha>0$,
  \begin{enumerate}
    \item Assume $\gamma\leq\frac{1}{4}$. If
    \begin{align*}
      n>\frac{2\big[(2+\alpha)\log{p}+2\log{|\mathcal{X}|}\big]}{\gamma^2},
    \end{align*}
    then $\forall i, j\in V, \forall x_i, x_j$,
    \begin{align*}
    |\hat{P}(x_i|x_j)-P(x_i|x_j)|<4\gamma
    \end{align*}
    with probability $1-\frac{c_1}{p^{\alpha}}$ for some constant $c_1$.
    \item Assume $\forall S\subset V, |S|\leq
    D_1+D_2+1, P(x_S)>\delta$ for some constant $\delta$, and $\gamma\leq \frac{\delta}{2}$. If
    \begin{align*}
      n>\frac{2\big[(1+\alpha)\log p+(D_1+D_2+1)\log L+(D_1+D_2+2)\log|\mathcal{X}|\big]}{\gamma^2},
    \end{align*}
    then $\forall i\in V, \forall j\in L_i, \forall S\subset L_i, |S|\leq
    D_1+D_2, \forall x_i, x_j, x_S$,
    \begin{align*}
      &|\hat{P}(x_i| x_j, x_S)-P(x_i|x_j, x_S)|<\frac{2\gamma}{\delta}
    \end{align*}
    with probability $1-\frac{c_2}{p^{\alpha}}$ for some constant $c_2$.
    \item Assume $\gamma\leq\frac{1}{2|\mathcal{X}|^{D_1+D_2+2}}<1$. If
    \begin{align*}
      n>\frac{2\big[(1+\alpha)\log p+(D_1+D_2+1)\log L+(D_1+D_2+2)\log|\mathcal{X}|\big]}{\gamma^2},
    \end{align*}
    then $\forall i, j\in V, |S|\leq D_1+D_2, \forall x_i, x_j, x_S$,
    \begin{align*}
      &|\hat{I}(X_i;X_j|X_S)-I(X_i;X_j|X_S)|<8|\mathcal{X}|^{D_1+D_2+2}\sqrt{\gamma}
    \end{align*}
    with probability $1-\frac{c_3}{p^{\alpha}}$ for some constant $c_3$,
  \end{enumerate}
\end{lem}
\begin{proof}
  See Appendix~\ref{appendix:concentration}.
\end{proof}

This lemma could be used as a guideline on how to choose between the two conditional independence tests for our algorithm to get lower sample complexity. The key difference is the dependence on the constant $\delta$, which is a lower bound on the probability of any $x_S$ with the set size $|S|\leq D_1+D_2+1$. The probability test requires a constant $\delta>0$ to achieve sample complexity $n = O(\log p)$, while the mutual information test does not depend on $\delta$ and also achieves sample complexity $n = O(\log p)$. We note that, while both tests have $O(\log p)$ sample complexity, the constants hidden in the order notation may be different for the two tests. For Ising models on bounded degree graphs, we show in the next section that a constant $\delta>0$ exists, and the probability test gives a lower sample complexity. On the other hand, for Ising models on the \er random graph $\Gp$, we could not get a constant $\delta>0$ as the maximum degree of the graph is unbounded, and the mutual information test gives a lower sample complexity.

\section{Computational Complexity for General Ising Models}\label{sec:complexity_general}

In this section, we apply our algorithm to the Ising models in Section~\ref{sec:loose}. We evaluate both the number of samples required to recover the graph with high probability and the running time of our algorithm. The results below are simple combinations of the results in the previous two sections. Unless otherwise stated, we assume that the edge coefficients $J_{ij}$ are bounded, i.e., $\Jmin \leq |J_{ij}| \leq \Jmax$. Throughout this section, we use the notation $x\wedge y$ to denote the minimum of $x$ and $y$.

\subsection{Bounded Degree Graph}

We assume the graph has maximum degree $d$. First we have the following lower bound on the probability of any finite size set of variables.

\begin{lem}\label{lem:bdd_a3}
  $\forall S\subset V, \forall x_S$, $P(x_S)\geq 2^{-|S|}\exp(-2(|S|+d)|S|J_{\max})$.
\end{lem}
\begin{proof}
  See Appendix~\ref{appendix:bdd}.
\end{proof}

Our algorithm with the probability test for the bounded degree graph case reproduces the algorithm in \cite{Guy}. For completeness, we state the following result without a proof since it is nearly identical to the result in \cite{Guy}, except for some constants.

\begin{corollary}\label{thm:bdd_alg_general}
  Let $\epsilon$ be defined as in Proposition~\ref{prop:bdd_loose}. Define
  \begin{align*}
    \delta = 2^{-2d}\exp(-12d^2J_{\max}).
  \end{align*}
  Let $\gamma = \frac{\epsilon\delta}{16}\ \wedge\ \frac{\delta}{2}<1$.
  If $n>\frac{2\big[(2d+1+\alpha)\log p+(2d+1)\log2\big]}{\gamma^2}$,
  the algorithm \ $\alg_P(d, d-1, \epsilon_2)$ recovers $G$ with probability $1-\frac{c}{p^{\alpha}}$ for some constant
  $c$. The running time of the algorithm is $O(np^{2d+1})$. \qed
\end{corollary}

\subsection{Bounded Degree Graph, Correlation Decay and Large Girth}
We assume the graph has maximum degree $d$. We also assume that the Ising model is in the correlation decay regime, i.e., $(d-1)\tanh\Jmax<1$, and the graph has large girth. Combining Theorem~\ref{thm:girth_loose}, Fact~\ref{result:alg1} and Lemma~\ref{lem:concentration}, We can show that the algorithm $\alg_P(1, 0, \epsilon)$ recovers the graph correctly with high probability for some constant $\epsilon$, and the running time is $O(np^3)$ for $n = O(\log p)$.

We can get even lower computational complexity using our second algorithm. The key observation is that, as there is no short path other than the direct edge between neighboring nodes, the correlation over the edge dominates the total correlation hence the pairwise non-degeneracy condition is satisfied. We note that the length of the second shortest path between neighboring nodes is no less than $g-1$.

\begin{lem}\label{lem:girth_correlation}
  Assume that $(d-1)\tanh{J_{\max}}<1$, and the girth $g$ satisfies
  \begin{align*}
    \beta \alpha^{g-1}\leq A\wedge \ln 2, 
  \end{align*}
  where $A = \frac{1}{1800}(1-e^{-4J_{\min}})$. Let $\epsilon' = 48A$. $\forall (i, j)\in E$, we have
  \begin{align*}
    \displaystyle\max_{x_i, x_j, x_j'}|P(x_i|x_j)-P(x_i|x_j')|>\epsilon'.
  \end{align*}
\end{lem}
\begin{proof}
  See Appendix~\ref{appendix:bdd}.
\end{proof}

Using this lemma, we can apply our second algorithm to learn the graph. Using Lemma~\ref{lem:bdd_correlation_decay}, if node $j$ is of distance $l_{\epsilon'} = \frac{\ln\frac{4\prealpha}{\epsilon'}}{\ln\frac{1}{\alpha}}$ hops from node $i$, we have
\begin{align*}
  \max_{x_i, x_j, x_j'}|P(x_i|x_j)-P(x_i|x_j')|<\prealpha\alpha^{l_{\epsilon'}}\leq
  \frac{\epsilon'}{4}.
\end{align*}
Therefore, in the correlation test, $L_i$ only includes nodes within distance $l_{\epsilon'}$ from $i$ and the
size
$|L_i|\leq d^{l_{\epsilon'}}$ since the maximum degree is $d$; i.e., $L = \max_i |L_i|\leq
d^{l_{\epsilon'}}$, which is a constant independent of $p$. Combining the previous lemma, Theorem~\ref{thm:girth_loose}, Fact~\ref{result:alg2} and Lemma~\ref{lem:concentration}, we get the following result.

\begin{corollary}\label{thm:girth_alg}
  Assume $(d-1)\tanh{J_{\max}}<1$. Assume $g, \epsilon$ and $\epsilon'$ satisfy Theorem~\ref{thm:girth_loose} and Lemma~\ref{lem:girth_correlation}. Let $\delta$ be defined as in Theorem~\ref{thm:bdd_alg_general}.
  Let $\gamma = \frac{\epsilon'}{32}\wedge\frac{\epsilon\delta}{16}\wedge
  \frac{\delta}{2}$. If
  \begin{align*}
    n>\frac{2\big[(2+\alpha)\log{p}+2l_{\epsilon'}\log
    d+3\log2\big]}{\gamma^2},
  \end{align*}
  the algorithm $\algpre_P(1, 0, \epsilon, \epsilon')$ recovers $G$ with probability $1-\frac{c}{p^{\alpha}}$ for some constant
  $c$. The running time of the algorithm is $O(np^2)$. \qed
\end{corollary}

\subsection{\er Random Graph $\Gp$ and Correlation Decay}

We assume the graph $G$ is generated from the prior $\mathcal{G}(p, \frac{c}{p})$ in which each edge is in $G$ with probability $\frac{c}{p}$ and the average degree for each node is $c$. Because the random graph has unbounded maximum degree, we cannot lower bound for the probability of a finite size set of random variables by a constant, for all $p$. To get good sample complexity, we use the mutual information test in our algorithm. Combining Theorem~\ref{thm:random_loose_nonneighbor}, Theorem~\ref{thm:random_loose_neighbor}, Fact~\ref{result:alg1} and Lemma~\ref{lem:concentration}, we get the following result.

\begin{corollary}\label{thm:random_alg}
  Assume $c\tanh{J_{\max}}<1$. There exists a constant $\epsilon>0$ such that, for $\gamma = \left(\frac{\epsilon}{32^2}\right)^2\wedge\frac{1}{64}<1$, if $n>\frac{2\big[(5+\alpha)\log p+5\log2\big]}{\gamma^2}$, the algorithm $\alg_I(2, 1, \epsilon)$ recovers the graph $G$ almost always. The running time of the algorithm is $O(np^5)$.\qed
\end{corollary}

\subsection{Sample Complexity}
In this subsection, we briefly summarize the number of samples required by our algorithm. According to the results in this section and the next section, $C\log p$ samples are sufficient in general, where the constant $C$ depends on the parameters of the model. When the Ising model is on a bounded degree graph with maximum degree $d$, the constant $C$ is of order $\exp(-O(d+d^2\Jmax))$. In particular, if the Ising model is in the correlation decay regime, then $d\Jmax = O(1)$ and the constant $C$ is of order $\exp(-O(d))$. When the Ising model is on a \er random graph $\Gp$ and is in the correlation decay regime, then the constant $C$ is lower bounded by some absolute constant independent of the model parameters. 

\section{Computational Complexity for Ferromagnetic Ising Models}\label{sec:complexity_ferro}

Ferromagnetic Ising models are Ising models in which all the edge coefficients $J_{ij}$ are nonnegative. We say $(i, j)$ is an edge if $J_{ij}>0$. One important property of ferromagnetic Ising models is association, which characterizes the positive dependence among the nodes.

\begin{definition}\label{def:association}\cite{Association}
  We say a collection of random variables $X = \\(X_1, X_2, \dots, X_n)$ is associated, or
  the random vector $X$ is associated, if
  \begin{align*}
    \mathrm{Cov}(f(X), g(X))\geq 0
  \end{align*}
  for all nondecreasing functions $f$ and $g$ for which $\E f(X), \E g(X), \E f(X)g(X)$
  exist. \qed
\end{definition}

\begin{prop}\label{prop:association_ferro}
\cite{Association_Ising} The random vector $X$ of a ferromagnetic Ising model (possibly with external fields) is
associated. \qed
\end{prop}

A useful consequence of the Ising model being associated is as follows.
\begin{corollary}\label{cor:ferro_Ising}
  Assume $X$ is a zero field ferromagnetic Ising model. For any $i, j$, $P(X_i = 1, X_j = 1)\geq\frac{1}{4}\geq P(X_i = 1, X_j = -1)$.
\end{corollary}
\begin{proof}
  See Appendix~\ref{appendix:ferro}.
\end{proof}

Informally speaking, the edge coefficient $J_{ij}>0$ means that $i$ and $j$ are positively dependent over the edge. For any path between $i, j$, as all the edge coefficients are positive, the dependence over the path is also positive. Therefore, the non-direct paths between a pair of neighboring nodes $i, j$ make $X_i$ and $X_j$, which are positively dependent over the edge $(i, j)$, even more positively dependent. This observation has two important implications for our algorithm.
\begin{enumerate}
  \item We do not need to break the short cycles with a set $T$ in order to detect the edges, so the maximization in the algorithm can be removed.
  \item The pairwise non-degeneracy is always satisfied for some constant $\epsilon'$, so we can apply the correlation test to reduce the computational complexity.
\end{enumerate}

\subsection{Bounded Degree Graph}

We assume the graph has maximum degree $d$. We have the following non-degeneracy result for ferromagnetic Ising models.

\begin{lem}\label{lem:bdd_loose_ferro}
  $\forall (i, j)\in E, S\subset V\setminus\{i, j\}$ and $\forall x_S$,
  \begin{align*}
    \max_{x_i, x_j, x_j'}|P(x_i | x_j, x_S)-P(x_i|x_j', x_S)| \geq &
  \frac{1}{16}(1-e^{-4J_{\min}})e^{-4|N_S|J_{\max}}.
  \end{align*}
\end{lem}
\begin{proof}
  See Appendix~\ref{appendix:ferro}.
\end{proof}

The following theorem justifies the remarks after Corollary~\ref{cor:ferro_Ising} and shows that the algorithm with the preprocessing step $\algpre(d, 0, \epsilon, \epsilon')$ can be used to learn the graph, where $\epsilon, \epsilon'$ are obtained from the above lemma. Recall that $L_i$ is the candidate neighbor set of node $i$ after the preprocessing step and $L = \max_i|L_i|$.

\begin{thm}\label{thm:bdd_alg_ferro}
  Let
  \begin{align*}
    \epsilon = \frac{1}{16}(1-e^{-4J_{\min}})e^{-4d^2J_{\max}}, \quad \epsilon' = \frac{1}{16}(1-e^{-4\Jmin}),
  \end{align*}
  and $\delta$ be defined as in Theorem~\ref{thm:bdd_alg_general}. Let $\gamma = \frac{\epsilon'}{32}\wedge\frac{\epsilon\delta}{16}\wedge
  \frac{\delta}{2}$. If $$n>\frac{2\big[(1+\alpha)\log{p}+(d+1)\log L
    +(d+2)\log2\big]}{\gamma^2},$$
  the algorithm $\algpre_P(d, 0, \epsilon, \epsilon')$ recovers $G$ with probability $1-\frac{c}{p^{\alpha}}$ for some constant
  $c$. The running time of the algorithm is $O(np^2+npL^{d+1})$. If we further assume that $(d-1)\tanh{J_{\max}}<1$, then the running time of the algorithm is $O(np^2)$.
\end{thm}
\begin{proof}
  We choose $|S|\leq d$ and $T = \emptyset$ in our algorithm, and we have $|N_S|\leq d^2$ as the maximum degree is $d$. By Lemma~\ref{lem:bdd_loose_ferro}, we have
  \begin{align*}
    \max_{x_i, x_j, x_j', x_S}|P(x_i | x_j, x_S)-P(x_i|x_j', x_S)| \geq \epsilon
  \end{align*}
  for any $|S|\leq d$. Therefore, the Ising model is a $(d, 0, \epsilon)$-loosely connected MRF. Note that Lemma~\ref{lem:bdd_loose_ferro} is applicable to any set $S$ (not necessarily the set $S$ in the conditional independence test). Applying Lemma~\ref{lem:bdd_loose_ferro} again with $S = \emptyset$,
  we get the pairwise non-degeneracy condition
  \begin{align*}
    \max_{x_i, x_j, x_j'}|P(x_i | x_j)-P(x_i|x_j')| \geq \epsilon'.
  \end{align*}
  Combining Fact~\ref{result:alg2} and Lemma~\ref{lem:concentration}, we get the correctness of the algorithm. The running time is $O(np^2+npL^{d+1})$, which is at most $O(np^{d+2})$.

  When $(d-1)\tanh{J_{\max}}<1$, as the Ising model is in the correlation decay regime, $L = \max_i|L_i|\leq d^{l_{\epsilon'}}$ is a constant independent of $p$ as argued for Theorem~\ref{thm:girth_alg}. Therefore, the running time is only $O(np^2)$ in this case.
\end{proof}

\subsection{\er Random Graph $\Gp$ and Correlation Decay}

When the Ising model is ferromagnetic, the result for the random graph is similar to that of a deterministic graph. For each graph sampled from the prior distribution, the dependence over the edges is positive. If $i, j$ are neighbors in the graph, having additional paths between them makes them more positively dependent, so we do not need to block those paths with a set $T$ to detect the edge and set $D_2 = 0$. In fact, we can prove a stronger result for neighbor nodes than the general case. The following result also appears in \cite{Anima1}, but we are unable to verify the correctness of all the steps there and so we present the result here for completeness.

\begin{thm}\label{thm:random_loose_ferro_neighbor}
  $\forall i\in V, \forall j\in N_i$, let $S$ be any set with $|S|\leq 2$, then almost always
  \begin{align*}
    I(X_i;X_j|X_S) = \Omega(1).
  \end{align*}
\end{thm}
\begin{proof}
  See Appendix~\ref{appendix:random}.
\end{proof}

Moreover, the pairwise non-degeneracy condition in Theorem~\ref{thm:bdd_alg_ferro} also holds here. We can thus use algorithm $\algpre(2, 0, \epsilon, \epsilon')$ to learn the graph. Without the pre-processing step, our algorithm is the same as in \cite{Anima1}. We show in the following theorem that using the pre-processing step our algorithm achieves lower computational complexity in the order of $p$.

\begin{thm}
  Assume $c\tanh \Jmax<1$ and the Ising model is ferromagnetic. Let $\epsilon'$ be defined as in Theorem~\ref{thm:bdd_alg_ferro}. There exists a constant $\epsilon>0$ such that, for $\gamma = \frac{\epsilon_1}{32}\wedge\left(\frac{\epsilon_2}{512}\right)^2\wedge\frac{1}{32}<1$,
  if $n>\frac{2\big[(2+\alpha)\log p+3\log L+5\log2\big]}{\gamma^2}$, the algorithm $\algpre_I(2, 0, \epsilon, \epsilon')$ recovers the graph $G$ almost always. The running time of the algorithm is $O(np^2)$.
\end{thm}
\begin{proof}
  Combining Theorem~\ref{thm:random_loose_nonneighbor}, Theorem~\ref{thm:random_loose_neighbor}, Fact~\ref{result:alg2}, Lemma~\ref{lem:concentration} and Lemma~\ref{lem:bdd_loose_ferro}, we get the correctness of the algorithm.

  From Theorem~\ref{thm:random_correlation_decay} we know that if $j$ is more than $\gamma_p$ hops away from $i$, the correlation between them decays as $o(p^{-\kappa})$. For the constant threshold
  $\frac{\epsilon'}{2}$, these far-away nodes are excluded from the candidate neighbor set $L_i$ when $p$ is large. It is shown in the proof of \cite[Lemma 2.1]{rapidmixing} that for $\mathcal{G}(p,
  \frac{c}{p})$, the number of nodes in the $\gamma_p$-ball around $i$ is not large with high probability. More specifically, $\forall i\in V, |B(i, \gamma_p)| = O(c^{\gamma_p}\log p)$ almost always, where $B(i, \gamma_p)$ is the set of all nodes which are at most $\gamma_p$ hops away from $i.$ Therefore we get
  \begin{align*}
    L = \max_i|L_i|\leq |B(i, \gamma_p)| = O(c^{\gamma_p}\log p) = O(p^{\frac{1}{K}}\log p) =
    O(p^{\frac{1}{3}}).
  \end{align*}
  So the total running time of algorithm $\alg_I(2, 0, \epsilon, \epsilon')$ is $O(np^2+npL^3) =O(np^2)$.
\end{proof}

\section{Experimental Results}\label{sec:experiment}

In this section, we present experimental results to show that importance of the choice of a non-zero $D_2$ in correctly estimating the edges and non-edges of the underlying graph of a MRF.
We evaluate our algorithm $\alg_I(D_1, D_2, \epsilon)$, which uses the mutual information test and does not have the preprocessing step, for general Ising models on grids and random graphs as illustrated in Figure~\ref{fig:graphs}. In a single run of the algorithm, we first generate the graph $G = (V, E)$: for grids, the graph is fixed, while for random graphs, the graph is generated randomly each time. After generating the graph, we generate the edge coefficients uniformly from $[-\Jmax, -\Jmin]\cup[\Jmin, \Jmax]$, where $\Jmin = 0.4$ and $\Jmax = 0.6$. We then generate samples from the Ising model by Gibbs sampling. The sample size ranges from $400$ to $1000$. The algorithm computes, for each pair of nodes $i$ and $j$,
\begin{align*}
  \hat{I}_{ij} = \min_{|S|\leq D_1}\max_{|T|\leq D_2} \hat{I}(X_i; X_j|X_S, X_T)
\end{align*}
using the samples. For a particular threshold $\epsilon$, the algorithm outputs $(i, j)$ as an edge if $\hat{I}_{ij}>\epsilon$ and gets an estimated graph $\hat{G} = (V, \hat{E})$.
We select $\epsilon$ optimally for each run of the simulation, using the knowledge of the graph, such that the number of errors in $\hat{E}$, including both errors in edges and non-edges, is minimized. The performance of the algorithm in each case is evaluated by the probability of success, which is the percentage of the correctly estimated edges, and each point in the plots is an average over $50$ runs. We then compare the performance of the algorithm under different choices of $D_1$ and $D_2$.

\begin{figure*}[!ht]
  \centering{
  \subfloat{\includegraphics[width=0.33\textwidth]{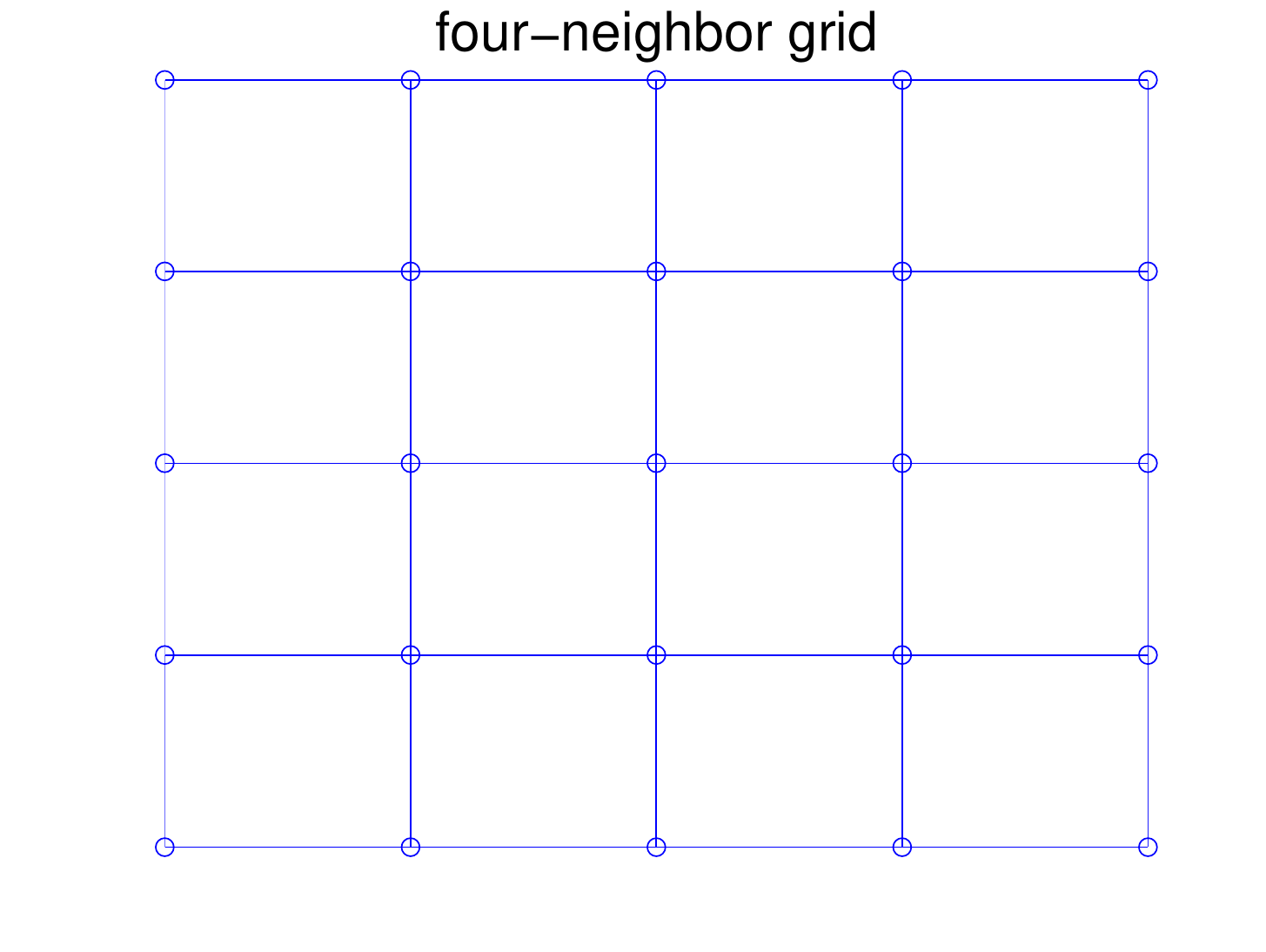}\label{fig:grid4}}\hfil \subfloat{\includegraphics[width=0.33\textwidth]{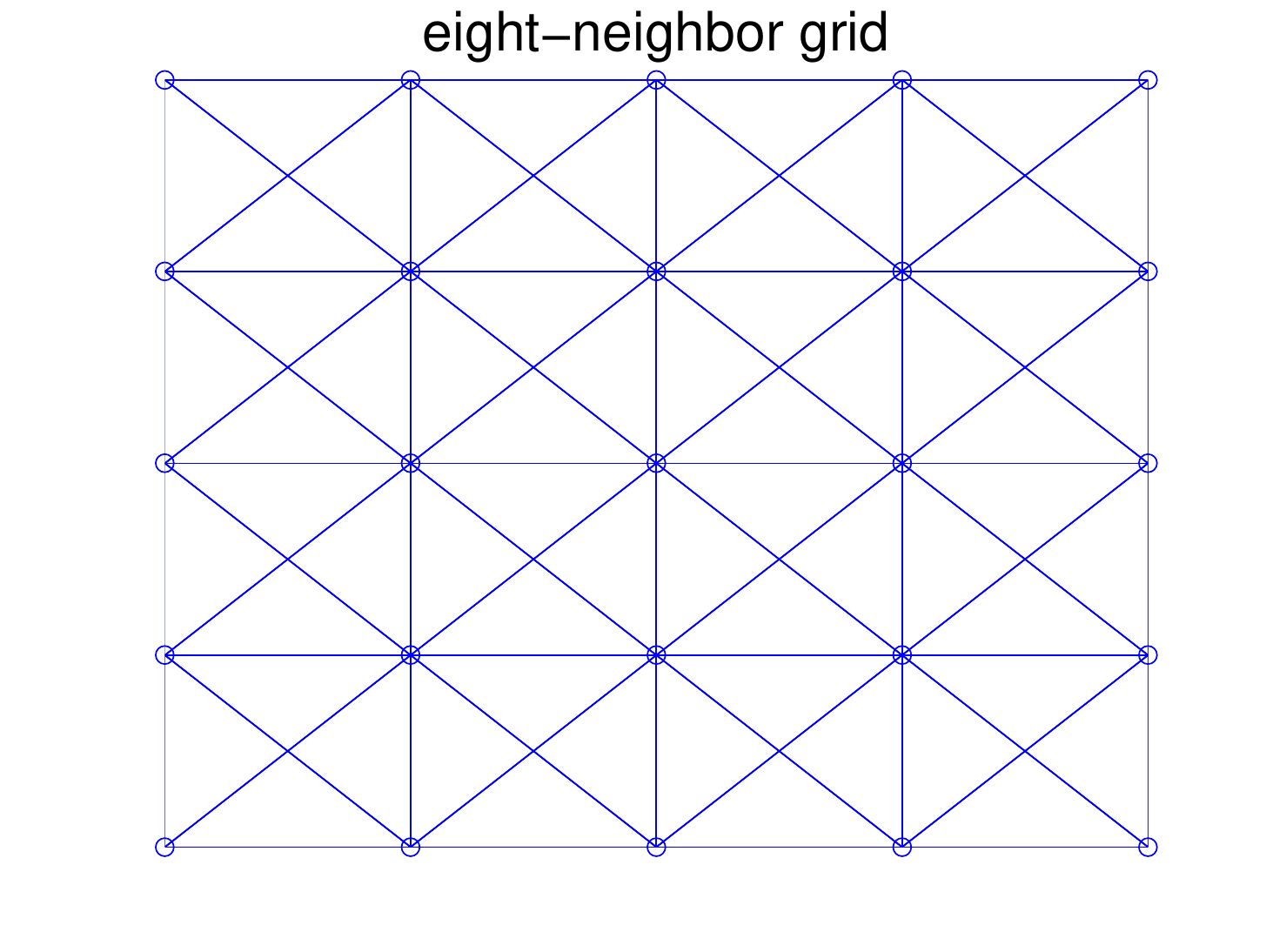}\label{fig:grid8}}\hfil \subfloat{\includegraphics[width=0.33\textwidth]{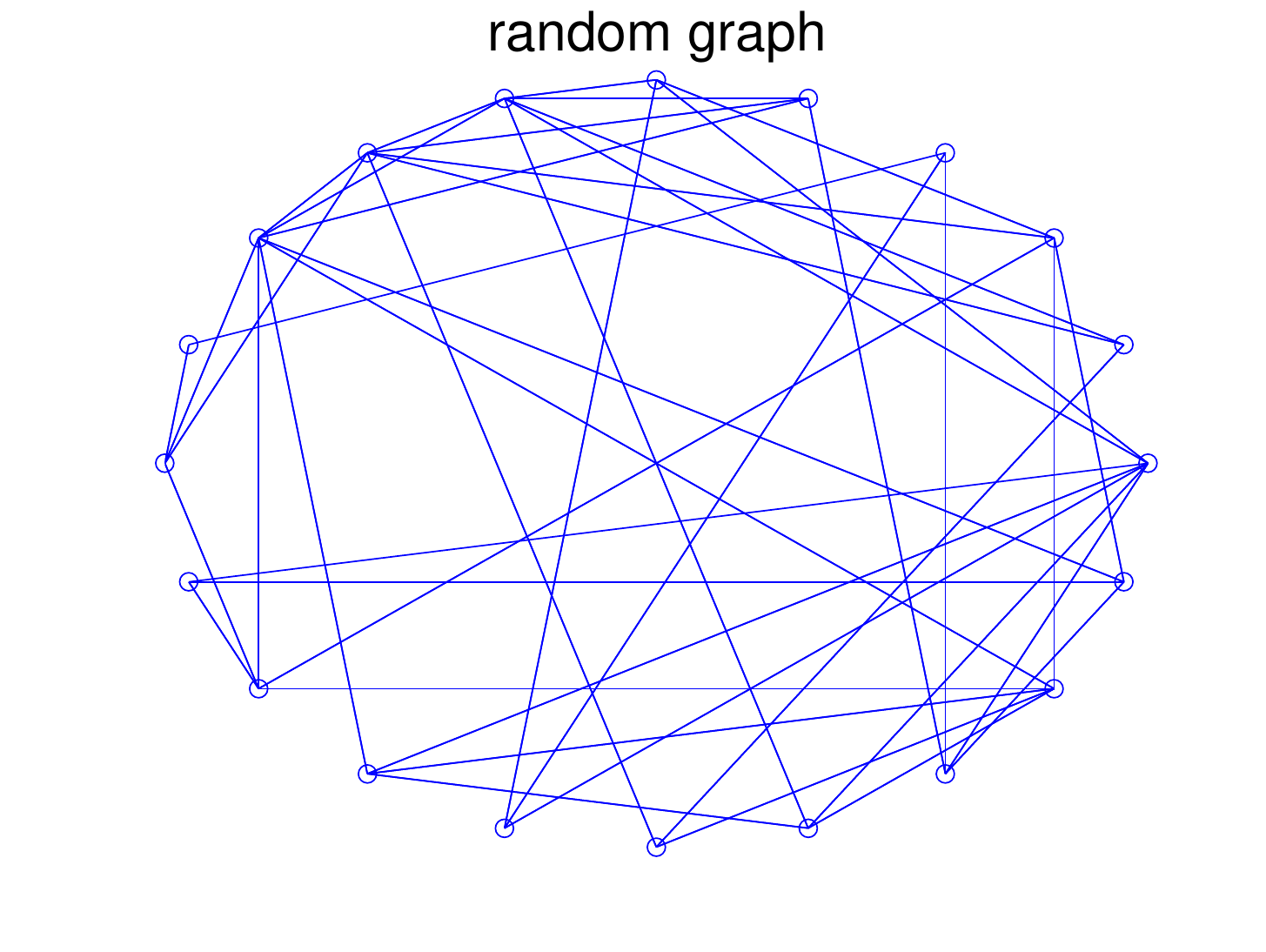}\label{fig:random}}}\\
  \caption{Illustrations of four-neighbor grid, eight-neighbor grid and the random graph. }\label{fig:graphs}
\end{figure*}

The experimental results for the algorithm with $D_1 = 0, \dots, 3$ and $D_2 = 0, 1$ applied to eight-neighbor grids on $25$ and $36$ nodes are shown in Figure~\ref{fig:grid_plot}. We omit the results for four-neighbor grids as the performances of the algorithm with $D_2 = 0$ and $D_2 >0$ are very close. In fact, four-neighbor grids do not have many short cycles and even the shortest non-direct paths are weak for the relatively small $\Jmax$ we choose, therefore there is no benefit using a set $T$ to separate the non-direct paths for edge detection. However, for eight-neighbor grids which are denser and have shorter cycles, the probability of success of the algorithm significantly improves by setting $D_2=1$, as seen from Figure~\ref{fig:grid_plot}. It is also interesting to note that increasing from $D_1 = 2$ to $D_1 = 3$ does not improve the performance, which implies that a set $S$ of size $2$ is sufficient to approximately separate the non-neighbor nodes in our eight-neighbor grids.

\begin{figure*}[!ht]
  \def\x{0.48}
  \centering{
  \subfloat{\includegraphics[width=\x\textwidth]{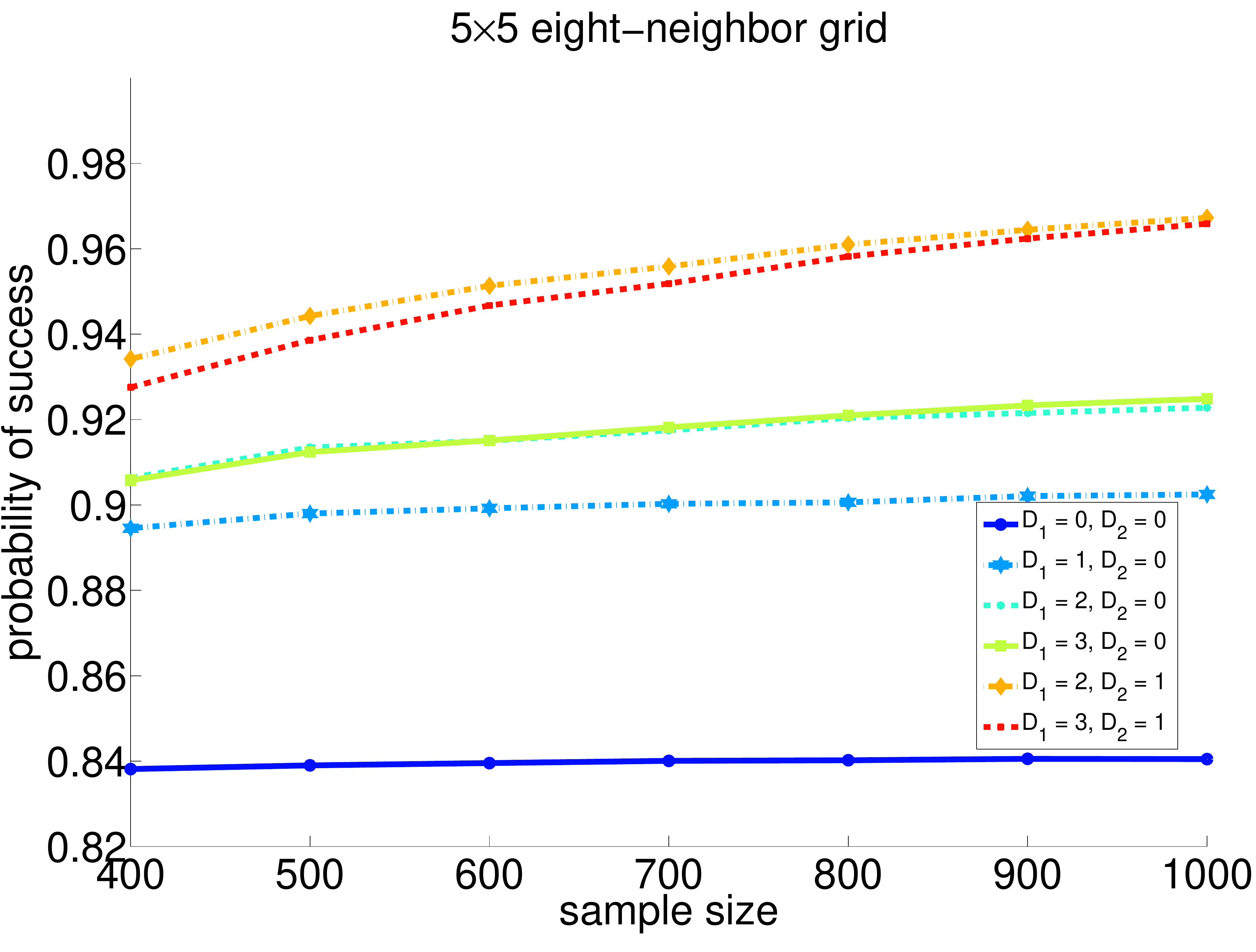}\label{fig:grid_plot_5}}\hfil \subfloat{\includegraphics[width=\x\textwidth]{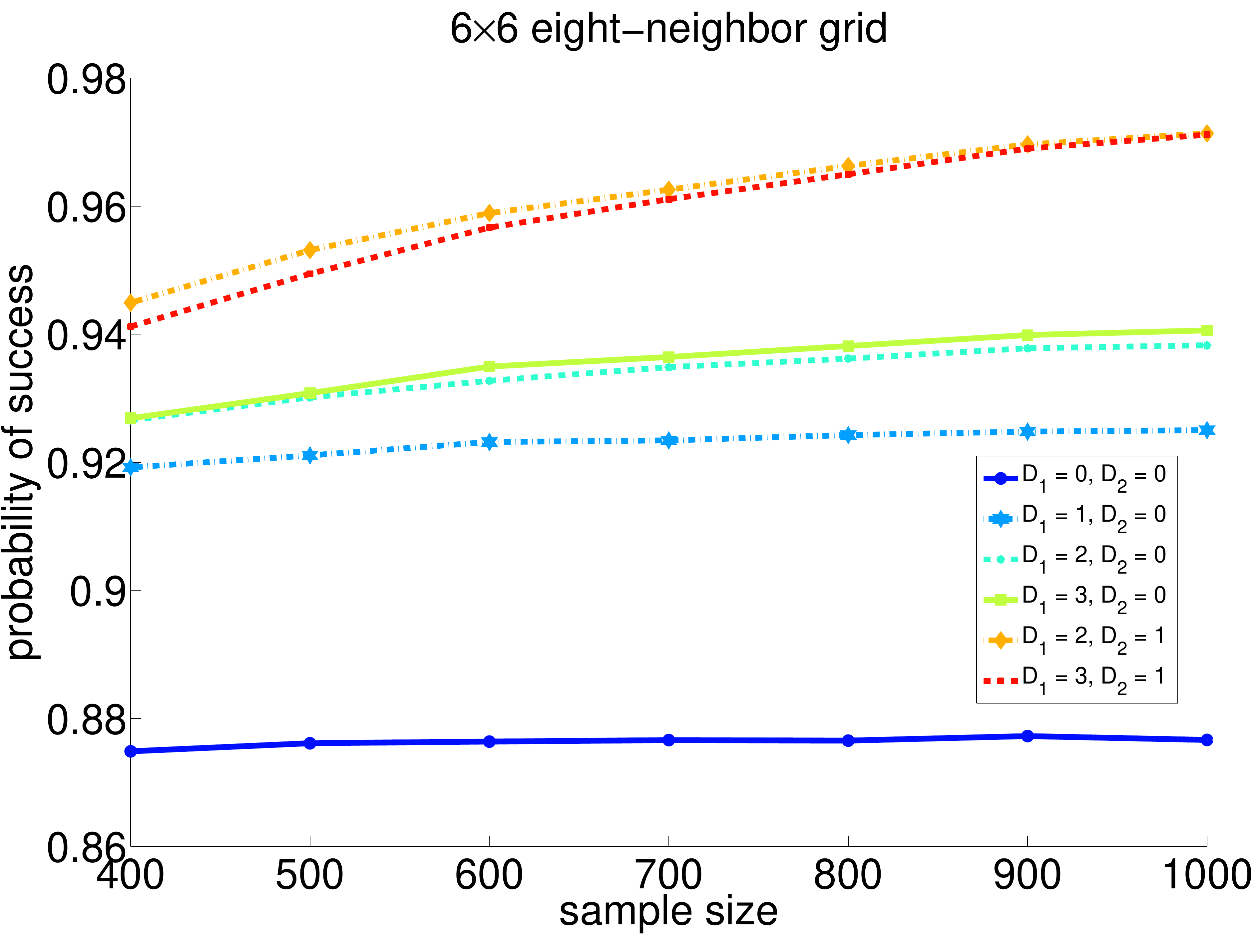}\label{fig:grid_plot_6}}}
  \caption{Plots of the probability of success versus the sample size for $5\times5$ and $6\times6$ eight-neighbor grids with $D_1 = 0, \dots, 3$ and $D_2 = 0, 1$.}\label{fig:grid_plot}
\end{figure*}

The experimental results for the algorithm with $D_1 = 0, \dots, 3$ and $D_2 = 0, 1$ applied to random graphs on $20$ and $30$ nodes are shown in Figure~\ref{fig:rand_plot}. For a random graph on $n$ nodes with average degree $d$, each edge is included in the graph with probability $\frac{d}{n-1}$ and is independent of all other edges. In the experiment, we choose average degree $5$ for the graphs on $20$ nodes and $7$ for the graphs on $30$ nodes. From Figure~\ref{fig:rand_plot}, the probability of success of the algorithm improves a lot when we increase $D_2$ from $0$ to $1$, which is very similar to the result of the eight-neighbor grids. We also note that, unlike the previous case, the algorithm with $D_1 = 3$ does have a better performance than with $D_1 = 2$ as there might be more short paths between a pair of nodes in random graphs.

\begin{figure*}[!ht]
  \def\x{0.48}
  \centering{
  \subfloat{\includegraphics[width=\x\textwidth]{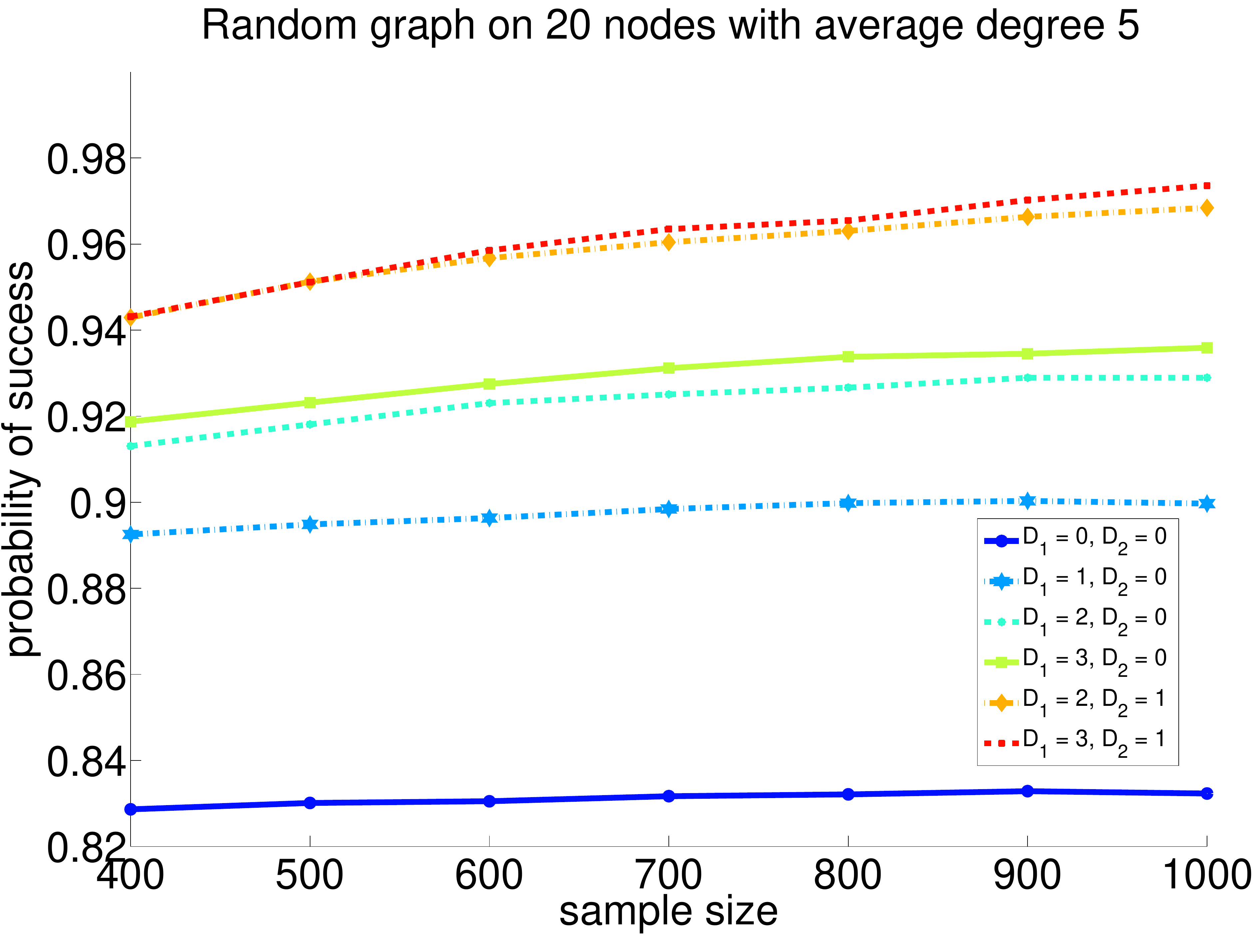}\label{fig:rand_plot_20}}\hfil \subfloat{\includegraphics[width=\x\textwidth]{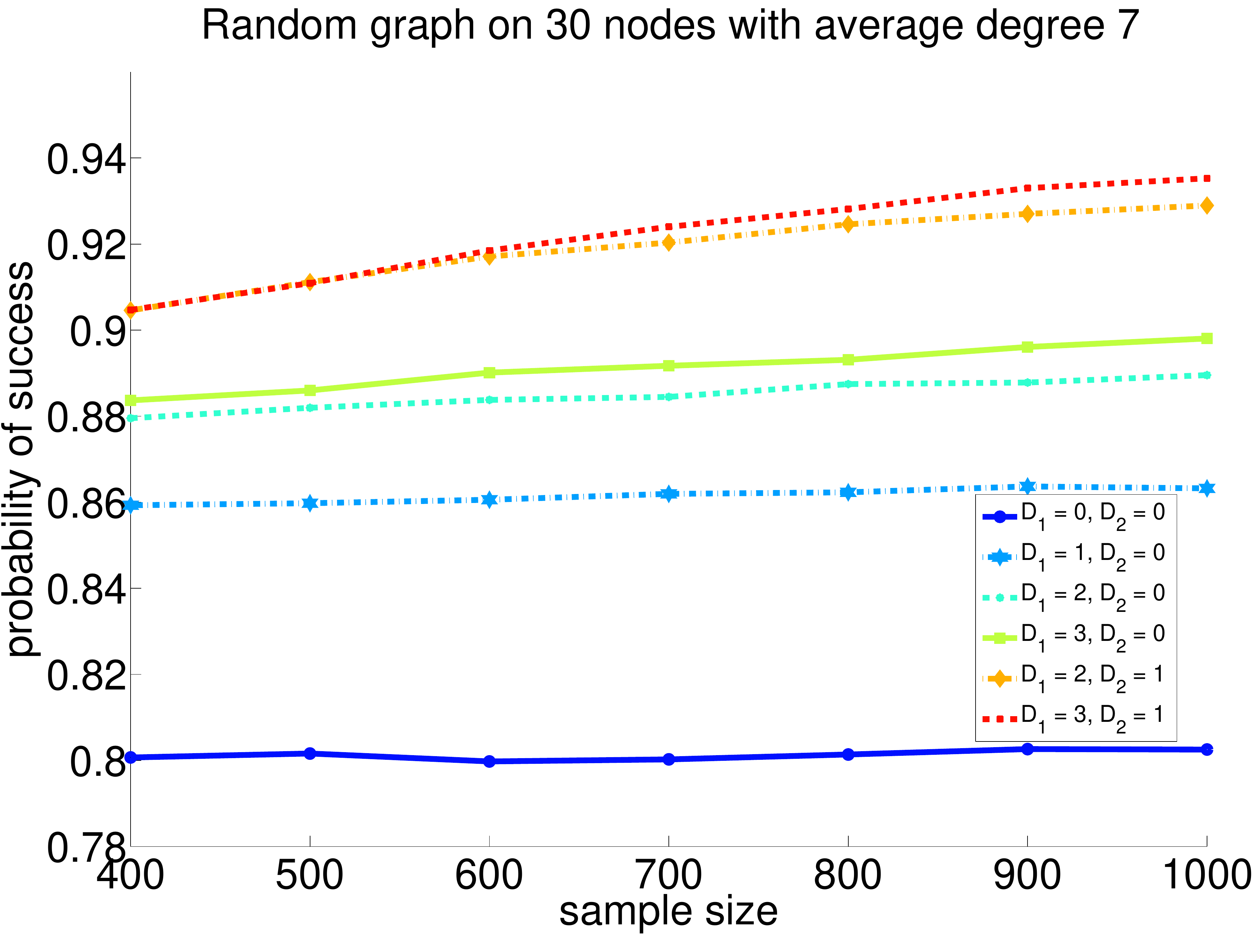}\label{fig:rand_plot_30}}}
  \caption{Plots of the probability of success versus the sample size for random graphs with $D_1 = 0, \dots, 3$ and $D_2 = 0, 1$.}\label{fig:rand_plot}
\end{figure*}

In a true experiment where only the data is available and no prior knowledge of the MRF is available, the choice of $\epsilon$ itself may affect the performance of the algorithm. At this time, we don not have any theoretical results to inform the choice of $\epsilon$. We briefly present a heuristic, which seems reasonable. However, extensive testing of the heuristic is required before we can confidently state that the heuristic is reasonable, which is beyond the scope of this paper. Our proposed heuristic is as follows.

For a given $D_1$ and $D_2$, we compute $\hat{I}_{ij}$ for each pair of nodes $i$ and $j$. If the choice of $D_1$ and $D_2$ is good, $\hat{I}_{ij}$ is expected to be close to $0$ for non-edges and away from $0$ for edges. Therefore, we can view the problem of choosing the threshold $\epsilon$ as a two-class hypothesis testing, where the non-edge class concentrates near $0$ while the edge class is more spread out. If we view $\hat{I}$, the collection of $\hat{I}_{ij}$ for all $i$ and $j$, as samples generated from the distribution of some random variable $Z$, then the hypothesis testing problem can viewed as one of finding the right $\epsilon$ such that the density of $Z$ has a big spike below $\epsilon$. One heuristic is to first estimate a smoothed density function from $\hat{I}$ via kernel density estimation \cite{Book_Stat} and then set $\epsilon$ to be the right boundary of the big spike near $0$.

In order to choose proper $D_1$ and $D_2$ for the algorithm, we can start with $(D_1, D_2) = (0, 0)$. At each step, we run the algorithm with two pairs of values $(D_1+1, D_2)$ and $(D_1, D_2+1)$ separately, and choose the pair that has a more significant change on the density estimated from $\hat{I}$ as the new value for $(D_1, D_2)$. We continue this process and stop increasing $D_1$ or $D_2$ if at some step there is no significant change for either pair of values.

Justifying this heuristic either through extensive experimentation or theoretical analysis is a topic for future research.

\section*{Acknowledgments} We thank Anima Anandkumar and Sreekanth Annapureddy for useful discussions. In particular, we would like to thank Anandkumar for suggesting the use of the SAW tree in the proof of Lemma~\ref{lem:random_cond_prob_lower_bound} and Annapureddy for suggesting the proof of Lemma~\ref{lem:edge}.

\appendix

\section{Bounded Degree Graph}\label{appendix:bdd}
\subsection{Proof of Lemma~\ref{lem:bdd_a3}}
  Let $N_S$ be the neighbor nodes of $S$. Note that each node in $S$ has at most $d$ neighbors in $N_S$. 
  \begin{align*}
    P(x_S) = &\sum_{x_{N_S}}{P(x_{N_S})P(x_S|x_{N_S})}\\
    \geq & \min_{x_S, x_{N_S}}P(x_S|x_{N_S})\\
     = & \min_{x_S,
     x_{N_S}}\frac{\exp(x_S^TJ_{SS}x_S+x_S^TJ_{SN_S}x_{N_S})}{\sum_{x_S'}\exp({x_S'}^TJ_{SS}x_S'+{x_S'}^TJ_{SN_S}x_{N_S})}\\
     \geq & \frac{\min_{x_S, x_{N_S}}\exp(x_S^TJ_{SS}x_S+x_S^TJ_{SN_S}x_{N_S})}{2^{|S|}\max_{x_S',
     x_{N_S}}\exp({x_S'}^TJ_{SS}x_S'+{x_S'}^TJ_{SN_S}x_{N_S})}\\
     \geq &
     \frac{\exp(-|S|^2J_{\max}-|S|dJ_{\max})}{2^{|S|}\exp(|S|^2J_{\max}+|S|dJ_{\max})}\\
     = & 2^{-|S|}\exp(-2(|S|+d)|S|J_{\max}). 
  \end{align*}

\subsection{Correlation Decay and Large Girth}

We assume that the Ising model on the bounded degree graph is further in the correlation decay regime. Both Theorem~\ref{thm:girth_loose} and Lemma~\ref{lem:girth_correlation} immediately follow from the following more general result, which characterizes the conditions under which the Ising model is $(D_1, D_2, \epsilon)$-loosely connected. We will make the connections at the end of this subsection. 

\begin{thm}\label{lem:sparse_loose}
  Assume $(d-1)\tanh{J_{\max}}<1$. Fix $D_1, D_2$.
  Let $h$ satisfy
  \begin{align*}
    \beta \alpha^{h}\leq A\wedge \ln 2,
  \end{align*}
  where $A = \frac{1}{1800}(1-e^{-4\Jmin})e^{-8(D_1+D_2)dJ_{\max}}$, and let $\epsilon =
  48Ae^{4(D_1+D_2)dJ_{\max}}$. Assume that there are at most $D_1$ paths shorter than $h$ between non-neighbor nodes and $D_2$ paths shorter than $h$ between neighboring nodes.
  Then $\forall (i, j)\in E$,
  $$
    \displaystyle\min_{\substack{S\subset V\setminus \{i\cup j\}\\|S|\leq D_1}}\ \max_{\substack{T\subset V\setminus \{i\cup j\}\\|T|\leq
    D_2}}\ \max_{x_i, x_j, x_j', x_S, x_T}|P(x_i|x_j, x_S, x_T)-P(x_i|x_j', x_S,
    x_T)|>\epsilon,
  $$
  and $\forall (i, j)\notin E$,
  $$
    \displaystyle\min_{\substack{S\subset V\setminus \{i\cup j\}\\|S|\leq D_1}}\ \max_{\substack{T\subset V\setminus \{i\cup j\}\\|T|\leq
    D_2}}\ \max_{x_i, x_j, x_j', x_S, x_T}|P(x_i|x_j, x_S, x_T)-P(x_i|x_j', x_S,
    x_T)|\leq\frac{\epsilon}{4}.
  $$
\end{thm}
\begin{proof}

  First consider $(i, j)\in E$. Without loss of generality, assume $J_{ij}>0$. By the assumption that there are at most $D_2$ paths shorter than $h$ between neighboring nodes, there exists $T'\subset N_i, |T'|\leq  D_2$ such that, when the set $T'$ is removed from the graph, the length of any path from $i$ to  $j$ is no less than $h$. For any $S$, let $T = T'\setminus S$. To simplify the
  notation, let $R = S\cup T$ and $W = V\setminus R$. For any value $x_R$, let $Q$ be the joint probability of $X_W$ conditioned on $X_R =  x_R$, i.e., $Q(X_W) = P(X_W|x_R)$. $Q$ has the same edge coefficients for the  unconditioned nodes, but is not zero-field as conditioning induces external fields. Let $\tilde{Q}$  denote the joint probability when edge $(i, j)$ is removed from $Q$. We note that $Q$ and $\tilde{Q}$ satisfy the same correlation  decay property as $P$, so
  \begin{align*}
    \tilde{Q}(1, 1) = &\tilde{Q}(X_i = 1)\tilde{Q}(X_j = 1|X_i = 1)\\
    \geq &\tilde{Q}(X_i = 1)[\tilde{Q}(X_j = 1|X_i = -1)-\prealpha\alpha^{l_{ij}}]\\
    \geq &\tilde{Q}(X_i = 1)[\tilde{Q}(X_j = 1|X_i = -1)-\prealpha\alpha^h]
  \end{align*}
  Similarly, $\tilde{Q}(-1, -1)\geq \tilde{Q}(X_i = -1)[\tilde{Q}(X_j = -1|X_i =  1)-\prealpha\alpha^h]$. Then,
  \begin{align*}
    &\tilde{Q}(1, 1)\tilde{Q}(-1, -1)\\
    \geq &\tilde{Q}(X_i = 1)\tilde{Q}(X_i = -1)[\tilde{Q}(X_j = 1|X_i = -1)-\prealpha\alpha^h]\\&[\tilde{Q}(X_j = -1|X_i =
    1)-\prealpha\alpha^h]\\
    \geq & \tilde{Q}(1, -1)\tilde{Q}(-1, 1)-2\prealpha\alpha^h
  \end{align*}
  Using the above inequality, we have the following lower bound on the $P$-test quantity.
  
  \begin{align*}
    &\max_{x_i, x_j, x_j'}|P(x_i | x_j, x_S, x_T)-P(x_i|x_j', x_S, x_T)|\\
    \geq &\left|Q(x_i = 1 |x_j = 1)-Q(x_i = 1 | x_j = -1)\right|\\
    = & \left|\frac{Q(x_i = 1, x_j = 1)}{Q(x_j = 1)}-\frac{Q(x_i = 1, x_j = -1)}{Q(x_j =
    -1)}\right|\\
    = & \left|\frac{Q(x_i = 1, x_j = 1)Q(x_i = -1, x_j = -1)-Q(x_i = 1, x_j = -1)Q(x_i = -1, x_j = 1)}{Q(x_j = 1)Q(x_j =
    -1)}\right|\\
    =  & \frac{\left|e^{2J_{ji}}\tilde{Q}(1, 1)\tilde{Q}(-1, -1)-e^{-2J_{ji}}\tilde{Q}(1, -1)\tilde{Q}(-1,
    1)\right|}{\left(e^{J_{ji}}\tilde{Q}(1, 1)+e^{-J_{ji}}\tilde{Q}(-1, 1)\right)\left(e^{-J_{ji}}\tilde{Q}(1, -1)+e^{J_{ji}}\tilde{Q}(-1, -1)\right)}\\
    \geq & e^{-2J_{ij}}\left[(e^{2J_{ij}}-e^{-2J_{ij}})\tilde{Q}(1, -1)\tilde{Q}(-1,
    1)-2e^{2J_{ij}}\prealpha\alpha^h\right]\\
    = &(1-e^{-4J_{ij}})\tilde{Q}(1, -1)\tilde{Q}(-1,
    1)-2\prealpha\alpha^h\\
    \geq &(1-e^{-4J_{\min}})\tilde{Q}(1, -1)\tilde{Q}(-1,
    1)-2\prealpha\alpha^h.
  \end{align*}

  Let $\check{Q}$ denote the joint probability when all the external field terms are removed from $\tilde{Q}$; i.e.,
  \begin{align*}
    \tilde{Q}(X_W)\propto\check{Q}(X_W)e^{h_W^TX_{W}}
  \end{align*}
  As there are at most $(D_1+D_2)d$ edges between $R$ and $W$, we have $||h_W||_1\leq (D_1+D_2)dJ_{\max}$. Hence, for any subset $U\subset W$ and value $x_U$,
  \begin{align*}
    \tilde{Q}(x_U) = &\frac{\tilde{Q}(x_U)}{\sum_{x_U'}{\tilde{Q}(x_U')}}\\
    = & \frac{\sum_{x_{W\setminus U}}{\check{Q}(x_U, x_{W\setminus U})e^{h_{W}^Tx_{W}}}}
    {\sum_{x_U'}{\sum_{x_{W\setminus U}'}{\check{Q}(x_U', x_{W\setminus U}')e^{h_{W}^Tx_{W}'}}}}\\
    \geq & \frac{\check{Q}(x_U)e^{-(D_1+D_2)dJ_{\max}}}{e^{(D_1+D_2)dJ_{\max}}}\\
    = &e^{-2(D_1+D_2)dJ_{\max}}\check{Q}(x_U).
  \end{align*}

  Moreover, $\check{Q}$ is zero-field by definition and again has the same correlation decay condition as $P$, hence
  \begin{align*}
    \check{Q}(1, -1)+\check{Q}(1, 1) = &\check{Q}(X_i = 1) = \frac{1}{2}\\
    \frac{\check{Q}(1, -1)}{\check{Q}(1, 1)}\geq &e^{-\prealpha\alpha^h},
  \end{align*}
  which gives the lower bound
  $
    \check{Q}(1, -1)\geq \frac{1}{2(1+e^{\prealpha\alpha^h})}.
  $
  Therefore, we have $$ \tilde{Q}(1, -1)\geq \frac{e^{-2(D_1+D_2)dJ_{\max}}}{2(1+e^{\prealpha\alpha^h})}.
  $$ The same lower bound applies for $\tilde{Q}(-1, 1)$. Hence,
  \begin{align*}
    &\max_{x_i, x_j, x_j'}|P(x_i | x_j, x_S, x_T)-P(x_i|x_j', x_S, x_T)|\\
    \geq & \frac{(1-e^{-4\Jmin})e^{-4(D_1+D_2)dJ_{\max}}}{4(1+e^{\prealpha\alpha^h})^2}-2\prealpha\alpha^h\\
    \geq & \frac{(1-e^{-4\Jmin})e^{-4(D_1+D_2)dJ_{\max}}}{36}-2\prealpha\alpha^h\\
    \geq & \frac{(1-e^{-4\Jmin})e^{-4(D_1+D_2)dJ_{\max}}}{36}-2e^{4(D_1+D_2)dJ_{\max}}\prealpha\alpha^h\\
    > & \epsilon.
  \end{align*}
  The second inequality uses the fact that $e^{\prealpha\alpha^h}<2$. The last inequality is by the choice of $h$.


  Next consider $(i, j)\notin E$. By the choice of $h$, there exists $S\subset N_i, |S|\leq  D_1$ such that, when the set $S$ is removed from the graph, the distance from $i$ to  $j$ is no less than $h$. Let $T$ set with $|T|\leq D_2$. As there is no edge between  $i, j$, the joint probability $Q$ and $\tilde{Q}$ are the same. Then $\forall x_S, x_T, x_i,  x_j, $
  \begin{align*}
    &|P(x_i|x_j, x_S, x_T)-P(x_i|-x_j, x_S, x_T)|\\
    = & |\tilde{Q}(x_i|x_j)-\tilde{Q}(x_i |-x_j)|\\
    = &\frac{|\tilde{Q}(x_i, x_j)\tilde{Q}(-x_i, -x_j)-\tilde{Q}(x_i, -x_j)\tilde{Q}(-x_i,
    x_j)|}{\tilde{Q}(x_j)\tilde{Q}(-x_j)}.
  \end{align*}
  Similar as above, we have
  $$ \tilde{Q}(x_j) \geq e^{-2(D_1+D_2)d\Jmax}\check{Q}(x_j) =  \frac{1}{2}e^{-2(D_1+D_2)d\Jmax}. $$
  The same bound applies for $\tilde{Q}(-x_j)$. Therefore,
  \begin{align*}
    &|P(x_i|x_j, x_S, x_T)-P(x_i|-x_j, x_S, x_T)|\\
    \leq &4e^{4(D_1+D_2)d\Jmax}|\tilde{Q}(x_i, x_j)\tilde{Q}(-x_i, -x_j)-\tilde{Q}(x_i, -x_j)\tilde{Q}(-x_i,
    x_j)|.
  \end{align*}
  By correlation decay and the fact $\prealpha\alpha^h<\ln 2<1$,
  \begin{align*}
    &Q(x_i, x_j)Q(-x_i, -x_j)\\ = & Q(x_i | x_j)Q(x_j)Q(-x_i | -x_j)Q(-x_j)\\
    \leq & (Q(x_i | -x_j)+\prealpha\alpha^h)Q(x_j)(Q(-x_i | -x_j)+\prealpha\alpha^h)Q(-x_j)\\
    \leq & Q(x_i, -x_j)Q(-x_i, x_j)+3\prealpha\alpha^h.
  \end{align*}
  Similarly, we have $Q(x_i, x_j)Q(-x_i, -x_j)\geq Q(x_i, -x_j)Q(-x_i,  x_j)-2\prealpha\alpha^h$. Hence, by the choice of $h$,
  \begin{align*}
    |P(x_i|x_j, x_S, x_T)-P(x_i|-x_j, x_S, x_T)|
    \leq &12e^{4(D_1+D_2)d\Jmax}\prealpha\alpha^h\leq \frac{\epsilon}{4}.
  \end{align*}

\end{proof}

Now we specialize this lemma for large girth graphs, in which there is at most one short path between non-neighbor nodes and no short non-direct path between neighboring nodes. Setting $D_1 = 1$ and $D_2=0$ in the theorem, we get Theorem~\ref{thm:girth_loose}. For the lower bound on the correlation between neighbor nodes, we set $D_1 = D_2 = 0$ in the theorem and get Lemma~\ref{lem:girth_correlation}.

\section{Ferromagnetic Ising Models}\label{appendix:ferro}
\subsection{Proof of Corollary~\ref{cor:ferro_Ising}}
  By Proposition~\ref{prop:association_ferro}, we apply Definition~\ref{def:association} to $X$ with $f(X) = X_i$ and $g(X) = X_j$, and get
  $
    \E [X_iX_j]\geq \E [X_i]\E [X_j].
  $
  As there is no external field, $P(X_i = 1) = P(X_i = -1)=0$ for any $i$ and $P(X_i = x_i, X_j = x_j) = P(X_i = -x_i, X_j = -x_j)$ for any $i, j$. Therefore, $\E[X_i] = 0$ and
  \begin{align*}
    \E[X_iX_j] = &4[P(X_i = 1, X_j = 1)-P(X_i = 1, X_j = -1)][P(X_i = 1, X_j = 1)\\&+P(X_i = 1, X_j = -1)].
  \end{align*}
  By the above inequality, noticing that $P(X_i = 1, X_j = 1)+P(X_i = 1, X_j = -1) = \frac{1}{2}$, we get the result.

\subsection{Proof of Lemma~\ref{lem:bdd_loose_ferro}}
  For any $i\in V, j\in N_i, S\subset V$, $Q, \tilde{Q}, \check{Q}$ are defined as in the
  proof of Lemma~\ref{lem:sparse_loose}.
  When $X$ is ferromagnetic but with external field, as in Corollary~\ref{cor:ferro_Ising}, we can show that
  \begin{align*}
    &P(X_i = 1, X_j = 1)P(X_i = -1, X_j = -1)\\
    \geq &P(X_i = 1, X_j = -1)P(X_i = -1, X_j =
    1)
  \end{align*}
  for any $i, j$. Therefore, we have
  \begin{align*}
    &\max_{x_i, x_j, x_j'}|P(x_i | x_j, x_S)-P(x_i|x_j', x_S)|\\
    \geq & e^{-2J_{ij}}\left|e^{2J_{ji}}\tilde{Q}(1, 1)\tilde{Q}(-1, -1)-e^{-2J_{ij}}\tilde{Q}(1, -1)\tilde{Q}(-1, 1)\right|\\
    \geq & e^{-2J_{ij}}(e^{2J_{ij}}-e^{-2J_{ij}})\tilde{Q}(1, 1)\tilde{Q}(-1, -1)\\
    \geq & (1-e^{-4J_{\min}})\tilde{Q}(1, 1)\tilde{Q}(-1, -1).
  \end{align*}

  We note that $\check{Q}$ is zero field, so by Corollary~\ref{cor:ferro_Ising} we get $\check{Q}(1, 1) = \check{Q}(-1, -1)\\\geq \frac{1}{4}$.
  As shown in Lemma~\ref{lem:sparse_loose},
  \begin{align*}
    \tilde{Q}(1, 1)\geq e^{-2|N_S|J_{\max}}\check{Q}(1, 1)\geq
    \frac{1}{4}e^{-2|N_S|J_{\max}}.
  \end{align*}
  The same lower bound can be obtained for $\tilde{Q}(-1, -1)$. Plugging the lower bounds to
  the above inequality, we get the result.

\section{Random Graphs}\label{appendix:random}

The proofs in this section are related to the techniques developed in \cite{Anima1,Anima2}. The key differences
are in adapting the proofs for general Ising models, as opposed to ferromagnetic models. We point out
similarities and differences as we proceed with the section.

\subsection{Self-Avoiding-Walk Tree and Some Basic Results}

This subsection introduces the notion of a self-avoiding-walk (SAW) tree, first introduced in \cite{SAW}, and presents some properties of a SAW tree.
For an Ising model on a graph $G$, fix an ordering of all the nodes. We say dge $(i, j)$ is
larger (smaller resp.) than $(i, l)$ with respect to node $i$ if $j$ comes after (before resp.)
$l$ in the ordering. The SAW tree rooted at node $i$ is denoted as $\Tsaw$. It
is essentially the tree of self-avoiding walks originated from node $i$ except that the
terminal nodes closing a cycle are also included in the tree with a fixed value $+1$ or
$-1$. In particular, a terminal node is fixed to $+1$ (resp. $-1$) if the closing edge of
the cycle is larger (resp. smaller) than the starting edge with respect to the terminal
node. Let $A$ denote the set of all terminal nodes in $\Tsaw$ and $x_A$ denote the fixed
configuration on $A$. For set $S\subset V$, let $U(S)$ denote the set of all non-terminal
copies of nodes in $S$ in $\Tsaw$. Notice that there is a natural way to define
conditioning on $\Tsaw$ according to the conditioning on $G$; specifically, if node $j$
in graph $G$ is fixed to a certain value, the non-terminal copies of $j$ in tree $\Tsaw$
are fixed to the same value.

One important result is \cite[Theorem 7]{SAW_Ising}, motivated by \cite{SAW}, says that the conditional
probability of node $i$ on graph $G$ is the same as the corresponding conditional
probability of node $i$ on tree $\Tsaw$, which is easier to deal with.

\begin{prop}\label{prop:SAW_equality}
  Let $S$ be a subset of $V$. $\forall x_i, x_S, P(x_i|x_S; G) = \\P(x_i|x_{U(S)}, x_A;
  \Tsaw)$.
\end{prop}

Next we list some basic results which will be used in later proofs. First we have the
following lemma about the number of short paths between a pair of nodes from \cite{Anima1}. The second part of Theorem~\ref{thm:random_correlation_decay} is an immediate result of
this lemma.

\begin{lem}\label{lem:random_graph_separator_size}
  \cite{Anima1} For all $i, j\in V$, the number of paths shorter than $\gamma_p$ between nodes $i, j$ is at most
  $2$ almost always.
\end{lem}

Let $B(i, l; \Tsaw)$ be the set of nodes of distance $l$ from $i$ on the tree $\Tsaw$.
Recall that $A$ is the set of terminal nodes in the tree. Let $\tilde{A}$ be the subset of
$A$ that are of distance at most $\gamma_p$ from $i$. The size of $B(i, l; \Tsaw)$ and
$\tilde{A}$ are upper bounded as follows.

\begin{lem}\cite[Lemma 2.2]{rapidmixing}\label{lem:random_graph_sphere_size}
  For $1\leq l\leq a\log p$, where $0<a<\frac{1}{2\log c}$, we have
  \begin{align*}
    \max_i|B(i, l; \Tsaw)| = O(c^l \log p), \text{ almost always.}
  \end{align*}
\end{lem}

\begin{lem}\label{lem:random_graph_SAW_terminal_size}
  $\forall i\in V, |\tilde{A}|\leq 1 $ in $\Tsaw$ almost always.
\end{lem}
\begin{proof}
  Each terminal node in $\tilde{A}$ corresponds to a cycle connected to $i$ with the
  total length of the cycle and the path to $i$ at most $\gamma_p$.
  Let $OLO_l$ denote the subgraph consists of two connected circles with total length
  $l$. This structure has $l-1$ nodes and $l$ edges. Let $H = \{OLO_l,\ l\leq 2\gamma_p\}$ and $N_H$ denote the
  number of subgraphs containing an instance from $H$. Then it is equivalent to show that
  there is at most 1 such small cycle close to each node or $N_H = 0$ almost always.
  \begin{align*}
    \E[N_H]\leq & \sum_{l=1}^{2\gamma_p}{\binom{p}{l-1}(l-1)!(l-1)^2(\frac{c}{p})^l}
    \leq O(\sum_{l=1}^{2\gamma_p}{p^{-1}l^2c^l})\\
     = &O(p^{-1}\gamma_p^2c^{2\gamma_p})\leq
    O(p^{-\frac{1}{3}}) = o(1).
  \end{align*}
  So, $P(N_H\geq 1) = o(1)$.
\end{proof}

\subsection{Correlation Decay in Random Graphs}
This subsection is to prove the first part of Theorem~\ref{thm:random_correlation_decay} which
characterizes the correlation decay property of a random graph.

First we state a correlation decay property for tree graphs. This result shows that
having external fields only makes the correlation decay faster.

\begin{lem}\label{lem:random_graph_tree_correlation_decay}
Let $P$ be a general Ising model with external fields on a tree $T$. Assume $|J_{ij}|\leq
J_{\max}$. $\forall i, j\in T$,
\begin{align*}
  |P(x_i|x_j)-P(x_i|x_j')|\leq (\tanh{J_{\max}})^{d(i, j)}.
\end{align*}
\end{lem}
\begin{proof}
The basic idea in the proof is get an upper bound that does not depend on the external field. To do this,
we proceed as in the proof of Lemma 4.1 in \cite{externfield}. First, as noted in \cite{externfield},
w.l.o.g. assume the tree is a line from $i$ to $j$. Then, we prove the result by
  induction on the number of hops in the line.
  \begin{enumerate}
    \item $d(i, j) = 1$ or $j\in N_i$. The graph has only two nodes. We have
    \begin{align*}
      P(x_i|x_j) = \frac{e^{J_{ij}x_ix_j+h_ix_i}}{e^{J_{ij}x_j+h_i}+e^{-J_{ij}x_j-h_i}}.
    \end{align*}
    Hence,
    \begin{align*}
      |P(x_i|x_j)-P(x_i|x_j')| = &
      \frac{|e^{2J_{ij}}-e^{-2J_{ij}}|}{(e^{J_{ij}+h_i}+e^{-J_{ij}-h_i})(e^{-J_{ij}+h_i}+e^{J_{ij}-h_i})}\\
      = &\frac{|e^{2J_{ij}}-e^{-2J_{ij}}|}{e^{2J_{ij}}+e^{-2J_{ij}}+e^{2h_i}+e^{-2h_i}}
    \end{align*}
    This function is even in both $J_{ij}$ and $h_i$. Without loss of generality, assume
    $J_{ij}\geq 0, h_i\geq 0$. It is not hard to see that the RHS is maximized when $h_i
    = 0$. So
    \begin{align*}
      |P(x_i|x_j)-P(x_i|x_j')| \leq \tanh|J_{ij}|\leq \tanh{J_{\max}}.
    \end{align*}
    The inequality suggests that, when there is external field, the impact
    of one node on the other is reduced.
    \item Assume the claim is true for $d(i, j)\leq k$. For $d(i, j)=k+1$, pick any $l$ on
    the path from $i$ to $j$, and note that $X_i$ --- $X_l$ --- $X_j$ forms a Markov chain. Moreover, $d(i, l)\leq k$ and $d(l,
    j)\leq k$.
    \begin{align*}
      &|P(x_i|x_j)-P(x_i|x_j')|\\
      = &|\sum_{x_l}{P(x_i|x_l)P(x_l|x_j)}-\sum_{x_l}{P(x_i|x_l)P(x_l|x_j')}|\\
      = &|P(x_i|x_l)(P(x_l|x_j)-P(x_l|x_j'))+P(x_i|x_l')(P(x_l'|x_j)-P(x_l'|x_j'))|\\
      = &|(P(x_i|x_l)-P(x_i|x_l'))(P(x_l|x_j)-P(x_l|x_j'))|\\
      \leq & (\tanh{J_{\max}})^{d(i, l)}(\tanh{J_{\max}})^{d(l, j)} = (\tanh{J_{\max}})^{d(i, j)}
    \end{align*}
    The third equality follows by observing that $P(x_l|x_j)-P(x_l|x_j') =
    -(P(x_l'|x_j)-P(x_l'|x_j'))$. The last inequality is by induction.
  \end{enumerate}
\end{proof}

Writing the conditional probability on a graph as a conditional probability on the corresponding SAW tree, we can
apply the above lemma and show the correlation decay property for random graphs.

\begin{lem}\label{lem:random_graph_correlation_decay}
  Let $P$ be a general Ising model on a graph $G$. Fix $i\in V$. $\forall j\notin N_i$, let $S$
  be the set that separates the paths shorter than $\gamma$ between $i, j$ and $B = B(i, \gamma; \Tsaw)$ , then
  $\forall x_i, x_j, x_j', x_S$,
  \begin{align*}
    |P(x_i|x_j, x_S) - P(x_i|x_j', x_S)|\leq |B|(\tanh{J_{\max}})^{\gamma}.
  \end{align*}
\end{lem}
\begin{proof}
  Let $Z$ be the subset of $U(j)$ on $\Tsaw$ that is not separated by $U(S)$ from $i$.
  By the definition of $S$, $Z$ is of distance at least $\gamma$ from $i$. So the $\gamma$-sphere $B$
  separates $Z$ and $i$.
  \begin{align*}
    &|P(x_i|x_j, x_S) - P(x_i|x_j', x_S)|\\
    \stackrel{(a)}{=}& |P(x_i|x_{U(j)}, x_{U(S)}, x_A; \Tsaw) - P(x_i|x_{U(j)}', x_{U(S)}, x_A; \Tsaw)|\\
    \stackrel{(b)}{=}& |P(x_i|x_Z, x_{U(S)}, x_A; \Tsaw) - P(x_i|x_Z', x_{U(S)}, x_A; \Tsaw)|\\
    \stackrel{(c)}{=}& |\sum_{x_B}{P(x_i|x_B, x_{U(S)}, x_A; \Tsaw})P(x_B|x_Z,x_{U(S)}, x_A; \Tsaw) \\
    & - \sum_{x_B}{P(x_i|x_B, x_{U(S)}, x_A; \Tsaw)P(x_B|x_Z',x_{U(S)}, x_A; \Tsaw)}|\\
    \leq & \max_{x_B}P(x_i|x_B, x_{U(S)}, x_A; \Tsaw)-\min_{x_B}P(x_i|x_B, x_{U(S)}, x_A; \Tsaw)\\
    \stackrel{(d)}{=}& P(x_i|x_B^M, x_{U(S)}, x_A; \Tsaw)- P(x_i|x_B^m, x_{U(S)}, x_A; \Tsaw)\\
    \stackrel{(e)}{\leq}& |B|(\tanh{J_{\max}})^{\gamma}.
  \end{align*}
  In the above, $(a)$ follows from the property of SAW tree in
  Prop~\ref{prop:SAW_equality}. Step $(b)$ is by the choice of $S$ and
  the definition of $Z$. Step $(c)$ uses the fact that $Z$ is separated from $i$ by $B$. In $(d)$, $x_B^M, x_B^m$
  represent the maximizer and minimizer respectively. Step $(e)$ is by telescoping the sign of
  $x_B$. Notice that the Hamming distance between $x_B^M, x_B^m$ is at most $|B|$, and we
  can apply the above lemma to each pair as the conditioning terms differ only on one
  node. The above proof is similar to the proof of Lemma 3 in \cite{Anima1}. However, in going from
  step (c) to step (d) above, it is important to note that our proof holds for general Ising models, whereas the
  proof in \cite{Anima1} is specific to ferromagnetic Ising models.
\end{proof}

\begin{proof}[Proof of Theorem~\ref{thm:random_correlation_decay}]
  As in \cite{Anima1}, setting $\gamma = \gamma_p$ in the above lemma and noticing that $$|B(i, \gamma_p; \Tsaw)| = O(c^{\gamma_p}\log
  p), $$ we get
  \begin{align*}
    &|P(x_i|x_j, x_S) - P(x_i|x_j', x_S)|\\
    \leq &O((c\tanh J_{\max})^{\gamma_p}\log p) =
    O(p^{-\frac{\log \alpha}{K\log c}}\log p) = o(p^{-\kappa}).
  \end{align*}
\end{proof}

\subsection{Asymptotic Lower Bound on $P(x_i |x_R)$ When $|R|\leq 3$}

This subsection is to prove that $P(x_i|x_R)$ is lower bounded by some constant when
$|R|\leq 3$. This result comes in handy when proving the other two theorems. This result
was conjectured to hold in \cite{Anima1} for ferromagnetic Ising models on the random
graph $\Gp$ without a proof. Here we prove that it is also true for general Ising models on
the random graph.

\begin{lem}\label{lem:random_cond_prob_lower_bound}
  $\forall i\in V, \forall R\subset V, |R|\leq 3$, there exists a constant $C$ such that
  $\forall x_i, x_R, P(x_i|x_R)\geq C$ almost always.
\end{lem}

This basic idea is that the conditional probability $P(x_i|x_R)$ is equal to some
conditional probability on a SAW tree, which in turn is viewed as some unconditional
probability on the same tree with induced external fields. Then we apply a tree reduction
to the SAW tree till only the root is left, and show that the induced external field on
the root is bounded, which implies that the probability of the root taking $+1$ or $-1$
is bounded.

On a tree graph, when calculating a probability which involves no nodes in a subtree, we can reduce
the subtree by simply summing (marginalizing) over all the nodes in it. This reduction produces an Ising
model on the rest part of the tree with the same $J_{ij}$ and $h_i$ except for the root
of the subtree, which would have an induced external field due to the reduction of the subtree. The
probability we want to calculate remains unchanged on this new tree. Such induced external fields
are bounded according to the following lemma.

\begin{lem}\label{lem:tree_extern_field}
  Consider a leaf node 2 and its parent node 1. The induced external field $h_1'$ on node
  1 due to summation over node 2 satisfies
  \begin{align*}
    |h_1'|\leq |h_2|\tanh|J_{12}|.
  \end{align*}
\end{lem}
We first prove an inequality which is used in the proof of the above lemma.

\begin{lem}\label{lem:extern_field_decay_inequality}
  $\forall x\geq 0, y\geq 0$,
  \begin{align*}
    e^{2x\tanh y}\geq \frac{e^{x+y}+e^{-x-y}}{e^{x-y}+e^{-x+y}}.
  \end{align*}
\end{lem}
\begin{proof}
  Let $u = \tanh y\in [0, 1)$, then $y = \frac{1}{2}\ln{\frac{1+u}{1-u}}$. The required result is equivalent to showing that
  \begin{align*}
    e^{2xu}[(1+u)e^{-x}+(1-u)e^x]>(1+u)e^x+(1-u)e^{-x}.
  \end{align*}
  Define
  \begin{align*}
    f_u(z) = (1+u)e^{uz}+(1-u)e^{(1+u)z}-(1+u)e^z-(1-u).
  \end{align*}
  Clearly, $f_u(0) = 0$, and
  \begin{align*}
    f_u'(z) = (1+u)[ue^{uz}+(1-u)e^{(1+u)z}-e^z].
  \end{align*}
  By the convexity of $e^z$, $ue^{uz}+(1-u)e^{(1+u)z}\geq e^z$. Hence, $f_u'(z)\geq 0$,
  which implies $f_u(z)\geq 0$. We finish the proof by noticing that the original
  inequality is equivalent to $f_u(2x)\geq 0$.
\end{proof}
\begin{proof}[Proof of Lemma~\ref{lem:tree_extern_field}]
  \begin{align*}
    \sum_{x_2}{e^{J_{12}x_1x_2+h_2x_2}} = e^{J_{12}x_1+h_2}+e^{-J_{12}x_1-h_2}\propto
    e^{h_1'x_1}.
  \end{align*}
  Comparing the ratio of $x_1 = \pm1$, we get
  \begin{align*}
     \frac{e^{J_{12}+h_2}+e^{-J_{12}-h_2}}{e^{-J_{12}+h_2}+e^{J_{12}-h_2}} =
     \frac{e^{h_1'}}{e^{-h_1'}} = e^{2h_1'}.
  \end{align*}
  So
  \begin{align*}
    h_1' =
    \frac{1}{2}\log{\frac{e^{J_{12}+h_2}+e^{-J_{12}-h_2}}{e^{-J_{12}+h_2}+e^{J_{12}-h_2}}}\leq
    |h_2|\tanh|J_{12}|.
  \end{align*}
  The last inequality follows from Lemma~\ref{lem:extern_field_decay_inequality}.
\end{proof}

It is easy to see that $|h_1'|\leq |h_2|\tanh|J_{\max}|<|h_2|$. By induction, we can
bound the external field induced by the whole subtree.

\begin{proof}[Proof of Lemma~\ref{lem:random_cond_prob_lower_bound}]
  First we have
  \begin{align*}
    P(x_i|x_R) = & P(x_i|x_{U(R)}, x_A; \Tsaw)\\
    = & \sum_{x_B}{P(x_i|x_B, x_{\tilde{U}(R)}, x_{\tilde{A}}; \Tsaw)P(x_B|x_{U(R)}, x_A; \Tsaw)}\\
    \geq & \min_{x_B}P(x_i|x_B, x_{\tilde{U}(R)}, x_{\tilde{A}}; \Tsaw)\\
    = &P(x_i|x_B^m, x_{\tilde{U}(R)}, x_{\tilde{A}}; \Tsaw) \triangleq Q(x_i),
  \end{align*}
  where $Q$ is the probability on the tree with external fields induced by $x_B^m, x_{\tilde{U}(R)},
  x_{\tilde{A}}$. We only need to consider the external fields on the parent nodes of $B,
  \tilde{U}(R), \tilde{A}$ as the conditional probability is on a tree. The
  nodes affected by $B$ are all $\gamma_p$ away from $i$ and the total number of them is no larger than
  $|B|$, which is bounded by Lemma~\ref{lem:random_graph_sphere_size}.
  The number of nodes affected by $\tilde{U}(R), \tilde{A}$ is no larger than
  $|\tilde{U}(R)|+|\tilde{A}|$. By Lemma~\ref{lem:random_graph_separator_size} and
  Lemma~\ref{lem:random_graph_SAW_terminal_size}, $|\tilde{U}(R)|\leq 2|R|$ and $|\tilde{A}|\leq
  1$ almost always. Applying the reduction technique to the tree till a single root node $i$,
  by Lemma~\ref{lem:tree_extern_field}, we bound the induced external field on $i$ as
  \begin{align*}
    |h_i|\leq & [(\tanh J_{\max})^{\gamma_n}|B|+(|\tilde{U}(R)|+|\tilde{A}|)]J_{\max}\\
    \leq & O((c\tanh J_{\max})^{\gamma_n}\log n+2|R|+1)\\
    \leq & O(n^{-\kappa}+7) = O(1).
  \end{align*}
  So,
  \begin{align*}
    Q(x_i) = \frac{e^{h_ix_i}}{e^{h_ix_i}+e^{-h_ix_i}}\geq \Omega(e^{-2|h_i|}) =
    \Omega(1).
  \end{align*}
  When $p$ is large enough, there exists some constant $C$ such that $P(x_i|x_R)\geq C$.
\end{proof}

\subsection{Proof of Theorem~\ref{thm:random_loose_nonneighbor}}

Let $S$ be the set that separates all the paths shorter than $\gamma_p$ between nodes $i, j$ with size $|S|\leq 3$. It is straightforward to show that $I(X_i; X_j|X_S) = o(p^{-2\kappa})$
in a manner similar to \cite[Lemma 5]{Anima1}. The only difference is
that the correlation decay property in Theorem~\ref{thm:random_correlation_decay} takes a
different form, which is easier to apply, therefore the proof there needs to be modified
accordingly. We also note that the constant $C$ in
Lemma~\ref{lem:random_cond_prob_lower_bound} is referred to as $f_{\min}(S)$ in
\cite{Anima1}. The details are omitted here.

\subsection{Proof of Theorem~\ref{thm:random_loose_neighbor}}

When $j$ is a neighbor of $i$, conditioned on the approximate separator $T$, there is one
copy of $j$ which is a child of the root $i$ in the SAW tree and is the only copy that within $\gamma_p$
from $i$. In Theorem~\ref{thm:random_loose_neighbor}, we show that the
effect of conditioning on $T$ is bounded and this copy of $j$ has a nontrivial impact on
$i$. With a little abuse of notation, we use $j$ to denote this copy of $j$ in $\Tsaw$. W.l.o.g assume $J_{ij}>0$. As in Lemma~\ref{lem:random_graph_correlation_decay},
  \begin{align*}
    &\max_{x_i, x_j}|P(x_i|x_j, x_T) - P(x_i|x_j', x_T)|\\
    =& \max_{x_i, x_j}|P(x_i|x_{U(j)}, x_{U(T)}, x_A; \Tsaw) - P(x_i|x_{U(j)}', x_{U(T)}, x_A; \Tsaw)|\\
    =& \max_{x_i, x_j}|P(x_i|x_Z, x_{U(T)}, x_A; \Tsaw) - P(x_i|x_Z', x_{U(T)}, x_A; \Tsaw)|\\
    =& \max_{x_i, x_j}|\sum_{x_B}{P(x_i|x_j, x_B, x_{\tilde{U}(T)}, x_{\tilde{A}}; \Tsaw)P(x_B|x_Z,x_{U(T)}, x_A; \Tsaw)} \\
    & - \sum_{x_B}{P(x_i|x_B, x_{\tilde{U}(T)}, x_{\tilde{A}}; \Tsaw)P(x_B|x_Z',x_{U(T)}, x_A; \Tsaw)}|\\
    \geq & \min_{x_B}P(x_i = +|x_j = +, x_B, x_{\tilde{U}(T)}, x_{\tilde{A}}; \Tsaw)\\&-\max_{x_B}P(x_i = +|x_j = -, x_B, x_{\tilde{U}(T)}, x_{\tilde{A}}; \Tsaw)\\
    =& P(x_i = +1|x_j = +1, x_B^m, x_{\tilde{U}(T)}, x_{\tilde{A}}; \Tsaw)\\&-P(x_i = +1|x_j = -1, x_B^M, x_{\tilde{U}(T)}, x_{\tilde{A}}; \Tsaw)\\
    = & P(x_i = +1|x_j = +1, x_B^m, x_{\tilde{U}(T)}, x_{\tilde{A}}; \Tsaw)\\&-P(x_i = +1|x_j = -1, x_B^m, x_{\tilde{U}(T)}, x_{\tilde{A}}; \Tsaw)\\
    &+ P(x_i = +1|x_j = -1, x_B^m, x_{\tilde{U}(T)}, x_{\tilde{A}}; \Tsaw)\\&-P(x_i = +1|x_j = -1, x_B^M, x_{\tilde{U}(T)}, x_{\tilde{A}}; \Tsaw)\\
    \geq & Q(x_i=+1|x_j =+1)-Q(x_i = +1|x_j = -1) - |B|(\tanh{J_{\max}})^{\gamma_n},
  \end{align*}
  where $Q$ is the probability measure on the reduced graph with only nodes $i, j$. We have
  \begin{align*}
    &Q(x_i=+1|x_j =+1)-Q(x_i = +1|x_j = -1)\\
    =&\frac{e^{2J_{ij}}-e^{-2J_{ij}}}{e^{2J_{ij}}+e^{-2J_{ij}}+e^{2h_i}+e^{-2h_i}}\\
    \geq
    &\frac{e^{2J_{\min}}-e^{-2J_{\min}}}{e^{2J_{\min}}+e^{-2J_{\min}}+e^{2h_i}+e^{-2h_i}}
    = \Omega(e^{-2|h_i|}).
  \end{align*}
  The external fields in $Q$ are induced by the conditioning on $B, \tilde{U}(T),
  \tilde{A}$. As in the proof of Lemma~\ref{lem:random_cond_prob_lower_bound}, we have
  $|h_i|\leq O(1)$,
  so $Q(x_i=+|x_j =+)-Q(x_i = +|x_j = -) = \Omega(1)$. Hence,
  \begin{align*}
    \max_{x_i, x_j}|P(x_i|x_j, x_S) - P(x_i|x_j', x_S)|\geq
    \Omega(1)-O(p^{-\kappa}) = \Omega(1).
  \end{align*}

  Using this result, the lower bound $I(X_i; X_j|X_T) = \Omega(1)$ simply follows from
  the proof of \cite[Lemma 7]{Anima1}. Again we note that the constant $C$ in
  Lemma~\ref{lem:random_cond_prob_lower_bound} is referred to as $f_{\min}(T)$ in
  \cite{Anima1}. The details are omitted here.

\subsection{Proof of Theorem~\ref{thm:random_loose_ferro_neighbor}}

The proof of the theorem needs the following lemma.

\begin{lem}
  $X$ is a ferromagnetic Ising model (possibly with external fields). $\forall i\in V,
  \forall S\subset V\setminus i$,
  \begin{align*}
    P(x_i = +1|x_S = +1)\geq P(x_i= +1|x_S = -1).
  \end{align*}
\end{lem}
\begin{proof}
  For any node $j\in S$, let probability $\tilde{P}(x_i, x_j) = P(x_i, x_j|x_{S\setminus
  j})$. The probability $\tilde{P}$ is still ferromagnetic and hence is associated. Then we have
  \begin{align*}
    &\tilde{P}(x_i = +1, x_j = +1)\tilde{P}(x_i = -1, x_j = -1)\\
    \geq &\tilde{P}(x_i = +1, x_j = -1)\tilde{P}(x_i = -1, x_j =
    +1).
  \end{align*}
  After some algebraic manipulation, we get
  \begin{align*}
    \tilde{P}(x_i = +1 | x_j = +1)\geq \tilde{P}(x_i = + 1| x_j = -1).
  \end{align*}
  This is equivalent saying that
  \begin{align*}
    P(x_i = +1|x_j = +1, x_{S\setminus j} = +1)\geq P(x_i = +1|x_j = -1, x_{S\setminus j} =
    +1).
  \end{align*}
  So flipping one node from $+1$ to $-1$ reduces the conditional probability regardless the
  what value the rest of the nodes take. Continuing this process till we flip all the
  nodes in $S$, we get the result
  \begin{align*}
    P(x_i = +1|x_S = +1)\geq P(x_i= +1|x_S = -1).
  \end{align*}
\end{proof}

\begin{proof}[Proof of Theorem~\ref{thm:random_loose_ferro_neighbor}]
  For $(i, j)\in E$, assume $J_{ij}>0$. Following the proof of
  Theorem~\ref{thm:random_loose_neighbor},
  \begin{align*}
    &\max_{x_i, x_j}|P(x_i|x_j, x_S) - P(x_i|x_j', x_S)|\\
    =& \max_{x_i, x_j}|P(x_i|x_{U(j)}, x_{U(S)}, x_A; \Tsaw) - P(x_i|x_{U(j)}', x_{U(S)}, x_A; \Tsaw)|\\
    \geq& P(x_i = +1|x_{\tilde{U}(j)} = +1, x_B^m, x_{\tilde{U}(S)}, x_{\tilde{A}}; \Tsaw)\\&-P(x_i = +1|x_{\tilde{U}(j)} = -1, x_B^M, x_{\tilde{U}(S)}, x_{\tilde{A}}; \Tsaw).
  \end{align*}

  The only difference here is that we might have more than one copy of $j$ in
  ${\tilde{U}(j)}$. Let $Z = {\tilde{U}(j)}\setminus j$. By the above lemma, we have
  \begin{align*}
    &\max_{x_i, x_j}|P(x_i|x_j, x_S) - P(x_i|x_j', x_S)|\\
    \geq & P(x_i = +1|x_j = +1, x_Z = +1, x_B^m, x_{\tilde{U}(S)}, x_{\tilde{A}}; \Tsaw)\\
    &-P(x_i = +1|x_j = -1, x_Z = +1, x_B^m, x_{\tilde{U}(S)}, x_{\tilde{A}}; \Tsaw)\\
    &+ P(x_i = +1|x_j = -1, x_Z = -1, x_B^m, x_{\tilde{U}(S)}, x_{\tilde{A}}; \Tsaw)\\
    &-P(x_i = +1|x_j = -1, x_Z = -1, x_B^M, x_{\tilde{U}(S)}, x_{\tilde{A}}; \Tsaw)\\
    \geq & Q(x_i=+1|x_j =+1)-Q(x_i = +1|x_j = -1) - |B|(\tanh{J_{\max}})^{\gamma_n}.
  \end{align*}
  As the size of $Z$ is only a constant, by the same reasoning, we finish the theorem.
\end{proof}


\section{Concentration}\label{appendix:concentration}

Before proving the concentration results in Lemma~\ref{lem:concentration}, we first present the following lemma which upper bounds the difference between the entropies of two distributions with their $l_1$-distance. Let $P$ and $Q$ be two probability mass functions on a discrete, finite set $\mathcal{X}$, and $H(P)$ and $H(Q)$ be their entropies respectively. The $l_1$ distance between $P$ and $Q$ is defined as $||P-Q||_1 = \sum_{x\in \mathcal{X}}|P(x)-Q(x)|$.

\begin{lem}\label{lem:entropy_bound}\cite[Theorem 17.3.3]{Cover}
  If $||P-Q||_1\leq\frac{1}{2}$, then $|H(P)-H(Q)|\leq
  -||P-Q||_1\log\frac{||P-Q||_1}{|\mathcal{X}|}$. When $||P-Q||_1\leq \frac{1}{e}$, the RHS is increasing in $||P-Q||_1$.
\end{lem}

\begin{proof}[Proof of Lemma~\ref{lem:concentration}]
By definition, $\forall S\subset V$ and $\forall x_S$, $|1_{\{X_S^{(i)}=\ x_S\}}-P(x_s)|\leq
1$ and
\begin{align*}
  \hat{P}(x_S) = \frac{1}{n}\sum_{i=1}^n{1_{\{X_S^{(i)}=\ x_S\}}}.
\end{align*}
By the Hoeffding inequality, 
\begin{align*}
  &P\left(|\hat{P}(x_S)-P(x_S)|\geq \gamma\right)\\
  = &P\left(\left|\sum_{i=1}^n{1_{\{X_S^{(i)}=\ x_S\}}}-nP(x_S)\right|\geq
  n\gamma\right)\leq 2e^{-\frac{n^2\gamma^2}{2n}}\leq 2e^{-\frac{n\gamma^2}{2}}.
\end{align*}
\begin{enumerate}
  \item By the union bound, we have
  \begin{align*}
    &P\left(\exists S\subset V, |S|\leq 2, \exists x_S, |\hat{P}(x_S)-P(x_S)|\geq
    \gamma\right)\\ <& p^2|\mathcal{X}|^22e^{-\frac{n\gamma^2}{2}} = 2e^{-\frac{n\gamma^2}{2}+2\log{p|\mathcal{X}|}}
  \end{align*}
  For our choice of $n$,  $\forall i, j\in V, \forall x_i, x_j$,
  \begin{align*}
    |\hat{P}(x_i, x_j)-P(x_i, x_j)|<\gamma, |\hat{P}(x_i)-P(x_i)|<\gamma,
  \end{align*}
  with probability $1-\frac{c_1}{p^{\alpha}}$ for some constant $c_1$,
  which gives $\hat{P}(x_j)>P(x_j)-\gamma\geq \frac{1}{2}-\gamma\geq\frac{1}{4}$ as
  $\gamma<\frac{1}{4}$. Hence,
  \begin{align*}
    &|\hat{P}(x_i| x_j)-P(x_i| x_j)|\\
    = &\frac{|\hat{P}(x_i, x_j)P(x_j)-P(x_i, x_j)\hat{P}(x_j)|}{P(x_j)\hat{P}(x_j)}\\
    \leq &\frac{\hat{P}(x_i, x_j)|P(x_j)-P(x_j)|}{P(x_j)\hat{P}(x_j)} +\frac{\hat{P}(x_j)|\hat{P}(x_i, x_j)-P(x_i, x_j)|}{P(x_j)\hat{P}(x_j)}\\
    \leq & \frac{2\gamma}{\frac{1}{2}} = 4\gamma.
  \end{align*}

  \item By the union bound, we have
  \begin{align*}
    &P\left(
    \begin{array}{c}
    \exists i\in V, \exists S\subset L_i, |S|\leq D_1+D_2+1, \exists x_S, \\
    |\hat{P}(x_S)-P(x_S)|\geq \gamma, |\hat{P}(x_i, x_S)-P(x_i, x_S)|\geq \gamma
    \end{array}
    \right)\\
    <& 2pL^{D_1+D_2+1}|\mathcal{X}|^{D_1+D_2+2}2e^{-\frac{n\gamma^2}{2}} \\
    < &4e^{-\frac{n\gamma^2}{2}+\log p+(D_1+D_2+1)\log L+(D_1+D_2+2)\log|\mathcal{X}|}.
  \end{align*}
  For our choice of $n$, $\forall i\in V, \forall j\in L_i, \forall S\subset L_i, |S|\leq
  D_1+D_2, \forall x_i, x_j, x_S$,
  \begin{align*}
    &|\hat{P}(x_i, x_j, x_S)-P(x_i, x_j, x_S)|\leq \gamma,\ |\hat{P}(x_j, x_S)-P(x_j, x_S)|\leq
    \gamma,
  \end{align*}
  with probability $1-\frac{c_2}{p^{\alpha}}$ for some constant $c_2$,
  which gives $\hat{P}(x_j, x_S)>P(x_j, x_S)-\gamma\geq \frac{\delta}{2}$ as
  $\gamma<\frac{\delta}{2}$. Hence,
  \begin{align*}
    &|\hat{P}(x_i| x_j, x_S)-P(x_i| x_j, x_S)|\\
    = &\frac{|\hat{P}(x_i, x_j, x_S)P(x_j, x_S)-P(x_i, x_j,
    x_S)\hat{P}(x_j, x_S)|}{P(x_j, x_S)\hat{P}(x_j, x_S)}\\
    \leq &\frac{\hat{P}(x_i, x_j, x_S)|P(x_j, x_S)-P(x_j, x_S)|}{P(x_j, x_S)\hat{P}(x_j, x_S)}\\
    &+\frac{\hat{P}(x_j, x_S)|\hat{P}(x_i, x_j, x_S)-P(x_i, x_j, x_S)|}{P(x_j, x_S)\hat{P}(x_j, x_S)}\\
    \leq & \frac{2\gamma}{\delta}.
  \end{align*}
  \item As in the previous case, for our choice of $n$,
  $\forall i, j\in V, \forall S\subset L_i, |S|\leq D_1+D_2, \forall x_i, x_j, x_S$,
  \begin{align*}
    |\hat{P}(x_i, x_j, x_S)-P(x_i, x_j, x_S)|\leq &\gamma, \\
    |\hat{P}(x_j, x_S)-P(x_j, x_S)|\leq & \gamma, \\
    |\hat{P}(x_S)-P(x_S)|\leq & \gamma
  \end{align*}
  with probability $1-\frac{c_3}{p^{\alpha}}$ for some constant $c_3$.
  So we get
  \begin{align*}
    ||\hat{P}(X_i, X_j, X_S)-P(X_i, X_j, X_S)||_1 \leq
    |\mathcal{X}|^{D_1+D_2+2}\gamma\leq \frac{1}{2}.
  \end{align*}
  By Lemma~\ref{lem:entropy_bound},
  \begin{align*}
    &|\hat{H}(X_i, X_j, X_S)-H(X_i, X_j, X_S)|\\
    \leq &-||\hat{P}(X_i, X_j, X_S)-P(X_i, X_j,
    X_S)||_1 \\&\log \frac{||\hat{P}(X_i, X_j, X_S)-P(X_i, X_j,
    X_S)||_1}{|\mathcal{X}|^{D_1+D_2+2}}\\
    \leq &-|\mathcal{X}|^{D_1+D_2+2}\gamma\log \gamma
    = -2 |\mathcal{X}|^{D_1+D_2+2}\gamma\log\sqrt{\gamma}\\
    \leq &2|\mathcal{X}|^{D_1+D_2+2}\sqrt{\gamma}.
  \end{align*}
  The last inequality used the fact that $0<-\sqrt{\gamma}\log\sqrt{\gamma}<1$ for $0<\gamma<1$.
  Similarly, we have the same upper bound for $|\hat{H}(X_i, X_S)-H(X_i, X_S)|, |\hat{H}(X_j, X_S)-H(X_j,
  X_S)|$ and $|\hat{H}(X_S)-H(X_S)|$. We finish the proof by noticing that
  \begin{align*}
    I(X_i;X_j|X_S) = H(X_i, X_S) + H(X_j, X_S)-H(X_i, X_j, X_S) - H(X_S).
  \end{align*}

\end{enumerate}
\end{proof}

\bibliographystyle{imsart-nameyear}
\bibliography{bibfile}

\begin{thebibliography}{22}

\bibitem[\protect\citeauthoryear{Alon and Spencer}{1992}]{Alon}
\begin{bbook}[author]
\bauthor{\bsnm{Alon},~\bfnm{Noga}\binits{N.}} \AND
  \bauthor{\bsnm{Spencer},~\bfnm{Joel~H.}\binits{J.~H.}}
(\byear{1992}).
\btitle{The Probabilistic Method}.
\bpublisher{Wiley}, \baddress{New York}.
\end{bbook}
\endbibitem

\bibitem[\protect\citeauthoryear{Anandkumar, Tan and Willsky}{2010}]{Anima1}
\begin{bmisc}[author]
\bauthor{\bsnm{Anandkumar},~\bfnm{Animashree}\binits{A.}},
  \bauthor{\bsnm{Tan},~\bfnm{Vincent}\binits{V.}} \AND
  \bauthor{\bsnm{Willsky},~\bfnm{Alan}\binits{A.}}
(\byear{2010}).
\btitle{High Dimensional Structure Learning of Ising Models on Sparse Random
  Graphs}.
\end{bmisc}
\endbibitem

\bibitem[\protect\citeauthoryear{Anandkumar, Tan and Willsky}{2011}]{Anima2}
\begin{barticle}[author]
\bauthor{\bsnm{Anandkumar},~\bfnm{Animashree}\binits{A.}},
  \bauthor{\bsnm{Tan},~\bfnm{Vincent Y.~F.}\binits{V.~Y.~F.}} \AND
  \bauthor{\bsnm{Willsky},~\bfnm{Alan~S.}\binits{A.~S.}}
(\byear{2011}).
\btitle{High-Dimensional Structure Estimation in Ising Models: Tractable Graph
  Families}.
\end{barticle}
\endbibitem

\bibitem[\protect\citeauthoryear{Banerjee, El~Ghaoui and
  d'Aspremont}{2008}]{Senate}
\begin{barticle}[author]
\bauthor{\bsnm{Banerjee},~\bfnm{Onureena}\binits{O.}},
  \bauthor{\bsnm{El~Ghaoui},~\bfnm{Laurent}\binits{L.}} \AND
  \bauthor{\bsnm{d'Aspremont},~\bfnm{Alexandre}\binits{A.}}
(\byear{2008}).
\btitle{Model Selection Through Sparse Maximum Likelihood Estimation for
  Multivariate Gaussian or Binary Data}.
\bjournal{J. Mach. Learn. Res.}
\bvolume{9}
\bpages{485--516}.
\end{barticle}
\endbibitem

\bibitem[\protect\citeauthoryear{Berger et~al.}{2005}]{externfield}
\begin{barticle}[author]
\bauthor{\bsnm{Berger},~\bfnm{Noam}\binits{N.}},
  \bauthor{\bsnm{Kenyon},~\bfnm{Claire}\binits{C.}},
  \bauthor{\bsnm{Mossel},~\bfnm{Elchanan}\binits{E.}} \AND
  \bauthor{\bsnm{Peres},~\bfnm{Yuval}\binits{Y.}}
(\byear{2005}).
\btitle{Glauber dynamics on trees and hyperbolic graphs}.
\bjournal{Probability Theory and Related Fields}
\bvolume{131}
\bpages{311-340}.
\bnote{10.1007/s00440-004-0369-4}.
\end{barticle}
\endbibitem

\bibitem[\protect\citeauthoryear{Boyd and Vandenberghe}{2004}]{Boyd}
\begin{bbook}[author]
\bauthor{\bsnm{Boyd},~\bfnm{Stephen}\binits{S.}} \AND
  \bauthor{\bsnm{Vandenberghe},~\bfnm{Lieven}\binits{L.}}
(\byear{2004}).
\btitle{Convex Optimization}.
\bpublisher{Cambridge University Press}, \baddress{New York, NY, USA}.
\end{bbook}
\endbibitem

\bibitem[\protect\citeauthoryear{Bresler, Mossel and Sly}{2008}]{Guy}
\begin{binproceedings}[author]
\bauthor{\bsnm{Bresler},~\bfnm{Guy}\binits{G.}},
  \bauthor{\bsnm{Mossel},~\bfnm{Elchanan}\binits{E.}} \AND
  \bauthor{\bsnm{Sly},~\bfnm{Allan}\binits{A.}}
(\byear{2008}).
\btitle{Reconstruction of Markov Random Fields from Samples: Some Observations
  and Algorithms}.
In \bbooktitle{APPROX-RANDOM}
\bpages{343-356}.
\end{binproceedings}
\endbibitem

\bibitem[\protect\citeauthoryear{Cover and Thomas}{1991}]{Cover}
\begin{bbook}[author]
\bauthor{\bsnm{Cover},~\bfnm{Thomas~M.}\binits{T.~M.}} \AND
  \bauthor{\bsnm{Thomas},~\bfnm{Joy~A.}\binits{J.~A.}}
(\byear{1991}).
\btitle{{Elements of information theory}}.
\bpublisher{Wiley-Interscience}, \baddress{New York, NY, USA}.
\end{bbook}
\endbibitem

\bibitem[\protect\citeauthoryear{Esary, Proschan and
  Walkup}{1967}]{Association}
\begin{barticle}[author]
\bauthor{\bsnm{Esary},~\bfnm{J.~D.}\binits{J.~D.}},
  \bauthor{\bsnm{Proschan},~\bfnm{F.}\binits{F.}} \AND
  \bauthor{\bsnm{Walkup},~\bfnm{D.~W.}\binits{D.~W.}}
(\byear{1967}).
\btitle{Association of Random Variables, with Applications}.
\bjournal{Annals of Mathematical Statistics}
\bvolume{38}
\bpages{1466--1473}.
\end{barticle}
\endbibitem

\bibitem[\protect\citeauthoryear{Hastie, Tibshirani and
  Friedman}{2003}]{Book_Stat}
\begin{bbook}[author]
\bauthor{\bsnm{Hastie},~\bfnm{T.}\binits{T.}},
  \bauthor{\bsnm{Tibshirani},~\bfnm{R.}\binits{R.}} \AND
  \bauthor{\bsnm{Friedman},~\bfnm{J.~H.}\binits{J.~H.}}
(\byear{2003}).
\btitle{The Elements of Statistical Learning},
\bedition{Corrected} ed.
\bpublisher{Springer}.
\end{bbook}
\endbibitem

\bibitem[\protect\citeauthoryear{Jung and Shah}{2006}]{SAW_Ising}
\begin{barticle}[author]
\bauthor{\bsnm{Jung},~\bfnm{Kyomin}\binits{K.}} \AND
  \bauthor{\bsnm{Shah},~\bfnm{Devavrat}\binits{D.}}
(\byear{2006}).
\btitle{Local approximate inference algorithms}.
\end{barticle}
\endbibitem

\bibitem[\protect\citeauthoryear{Liggett}{2010}]{Association_Ising}
\begin{barticle}[author]
\bauthor{\bsnm{Liggett},~\bfnm{Thomas~M.}\binits{T.~M.}}
(\byear{2010}).
\btitle{{Stochastic models for large interacting systems and related
  correlation inequalities}}.
\bjournal{Proceedings of the National Academy of Sciences}
\bvolume{107}
\bpages{16413--16419}.
\bdoi{10.1073/pnas.1011270107}
\end{barticle}
\endbibitem

\bibitem[\protect\citeauthoryear{Montanari and Pereira}{2009}]{Jose}
\begin{bincollection}[author]
\bauthor{\bsnm{Montanari},~\bfnm{Andrea}\binits{A.}} \AND
  \bauthor{\bsnm{Pereira},~\bfnm{Jose~Ayres}\binits{J.~A.}}
(\byear{2009}).
\btitle{Which graphical models are difficult to learn?}
In \bbooktitle{Advances in Neural Information Processing Systems 22}
(\beditor{\bfnm{Y.}\binits{Y.}~\bsnm{Bengio}},
  \beditor{\bfnm{D.}\binits{D.}~\bsnm{Schuurmans}},
  \beditor{\bfnm{J.}\binits{J.}~\bsnm{Lafferty}},
  \beditor{\bfnm{C.~K.~I.}\binits{C.~K.~I.}~\bsnm{Williams}} \AND
  \beditor{\bfnm{A.}\binits{A.}~\bsnm{Culotta}}, eds.)
\bpages{1303--1311}.
\end{bincollection}
\endbibitem

\bibitem[\protect\citeauthoryear{Mossel and Sly}{2008}]{rapidmixing}
\begin{binproceedings}[author]
\bauthor{\bsnm{Mossel},~\bfnm{Elchanan}\binits{E.}} \AND
  \bauthor{\bsnm{Sly},~\bfnm{Allan}\binits{A.}}
(\byear{2008}).
\btitle{Rapid mixing of Gibbs sampling on graphs that are sparse on average}.
In \bbooktitle{Proceedings of the nineteenth annual ACM-SIAM symposium on
  Discrete algorithms}.
\bseries{SODA '08}
\bpages{238--247}.
\end{binproceedings}
\endbibitem

\bibitem[\protect\citeauthoryear{Netrapalli et~al.}{2010}]{Sujay}
\begin{binproceedings}[author]
\bauthor{\bsnm{Netrapalli},~\bfnm{P.}\binits{P.}},
  \bauthor{\bsnm{Banerjee},~\bfnm{S.}\binits{S.}},
  \bauthor{\bsnm{Sanghavi},~\bfnm{S.}\binits{S.}} \AND
  \bauthor{\bsnm{Shakkottai},~\bfnm{S.}\binits{S.}}
(\byear{2010}).
\btitle{Greedy Learning of Markov Network Structure}.
In \bbooktitle{Allerton Conf. on Communication, Control and Computing}.
\end{binproceedings}
\endbibitem

\bibitem[\protect\citeauthoryear{Quinn, Kiyavash and Coleman}{2012}]{Negar}
\begin{barticle}[author]
\bauthor{\bsnm{Quinn},~\bfnm{Christopher~J.}\binits{C.~J.}},
  \bauthor{\bsnm{Kiyavash},~\bfnm{Negar}\binits{N.}} \AND
  \bauthor{\bsnm{Coleman},~\bfnm{Todd~P.}\binits{T.~P.}}
(\byear{2012}).
\btitle{Directed Information Graphs}.
\bjournal{CoRR}
\bvolume{abs/1204.2003}.
\end{barticle}
\endbibitem

\bibitem[\protect\citeauthoryear{Ravikumar, Wainwright and
  Lafferty}{2010}]{Martin}
\begin{barticle}[author]
\bauthor{\bsnm{Ravikumar},~\bfnm{P.}\binits{P.}},
  \bauthor{\bsnm{Wainwright},~\bfnm{M.~J.}\binits{M.~J.}} \AND
  \bauthor{\bsnm{Lafferty},~\bfnm{J.}\binits{J.}}
(\byear{2010}).
\btitle{{High-dimensional Ising model selection using $l_1$-regularized
  logistic regression}}.
\bjournal{Annals of Statistics}
\bvolume{38}
\bpages{1287--1319}.
\end{barticle}
\endbibitem

\bibitem[\protect\citeauthoryear{Ray, Sanghavi and Shakkottai}{2012}]{Sujay2}
\begin{binproceedings}[author]
\bauthor{\bsnm{Ray},~\bfnm{Avik}\binits{A.}},
  \bauthor{\bsnm{Sanghavi},~\bfnm{Sujay}\binits{S.}} \AND
  \bauthor{\bsnm{Shakkottai},~\bfnm{Sanjay}\binits{S.}}
(\byear{2012}).
\btitle{Greedy Learning of Graphical Models with Small Girth}.
In \bbooktitle{Allerton Conf. on Communication, Control and Computing}.
\end{binproceedings}
\endbibitem

\bibitem[\protect\citeauthoryear{Santhanam and
  Wainwright}{2012}]{Sample_Complexity}
\begin{barticle}[author]
\bauthor{\bsnm{Santhanam},~\bfnm{Narayana~P.}\binits{N.~P.}} \AND
  \bauthor{\bsnm{Wainwright},~\bfnm{Martin~J.}\binits{M.~J.}}
(\byear{2012}).
\btitle{Information-Theoretic Limits of Selecting Binary Graphical Models in
  High Dimensions}.
\bjournal{IEEE Transactions on Information Theory}
\bvolume{58}
\bpages{4117-4134}.
\end{barticle}
\endbibitem

\bibitem[\protect\citeauthoryear{Uhler et~al.}{2012}]{Faithful}
\begin{barticle}[author]
\bauthor{\bsnm{Uhler},~\bfnm{Caroline}\binits{C.}},
  \bauthor{\bsnm{Raskutti},~\bfnm{Garvesh}\binits{G.}},
  \bauthor{\bsnm{B\"uhlmann},~\bfnm{Peter}\binits{P.}} \AND
  \bauthor{\bsnm{Yu},~\bfnm{Bin}\binits{B.}}
(\byear{2012}).
\btitle{Geometry of faithfulness assumption in causal inference}.
\end{barticle}
\endbibitem

\bibitem[\protect\citeauthoryear{Weitz}{2006}]{SAW}
\begin{binproceedings}[author]
\bauthor{\bsnm{Weitz},~\bfnm{Dror}\binits{D.}}
(\byear{2006}).
\btitle{Counting independent sets up to the tree threshold}.
In \bbooktitle{Proceedings of the thirty-eighth annual ACM symposium on Theory
  of computing}.
\bseries{STOC '06}
\bpages{140--149}.
\end{binproceedings}
\endbibitem

\bibitem[\protect\citeauthoryear{Zhang, Liang and
  Bai}{2011}]{Correlation_Decay}
\begin{barticle}[author]
\bauthor{\bsnm{Zhang},~\bfnm{Jinshan}\binits{J.}},
  \bauthor{\bsnm{Liang},~\bfnm{Heng}\binits{H.}} \AND
  \bauthor{\bsnm{Bai},~\bfnm{Fengshan}\binits{F.}}
(\byear{2011}).
\btitle{Approximating partition functions of the two-state spin system}.
\bjournal{Inf. Process. Lett.}
\bvolume{111}
\bpages{702--710}.
\bdoi{http://dx.doi.org/10.1016/j.ipl.2011.04.012}
\end{barticle}
\endbibitem

\end{thebibliography}

\end{document}